\def\isarxiv{1}

\def\paperTitle{Toward Highly Efficient and Private Submodular Maximization via Matrix-Based Acceleration}

\def\paperRTitle{\paperTitle} 

\def\paperAuthor{
Boyu Liu \and
Lianke Qin\thanks{University of California, Santa Barbara.} \and
Zhao Song\thanks{\texttt{magic.linuxkde@gmail.com}. Simons Institute for the Theory of Computing, UC Berkeley.} \and
Yitan Wang\thanks{Yale University.} \and
Jiale Zhao\thanks{Guangdong University of Technology.}
}

\ifdefined\isarxiv
\documentclass[11pt]{article}
\usepackage[numbers]{natbib}
\else
\documentclass{article}

\fi

\ifdefined\isarxiv

\usepackage{amsmath}
\usepackage{amsthm}
\usepackage{amssymb}
\usepackage{algorithm}
\usepackage{subfig}
\usepackage{algpseudocode}
\usepackage{graphicx}
\usepackage{grffile}
\usepackage{wrapfig,epsfig}
\usepackage{url}
\usepackage{xcolor}
\usepackage{epstopdf}

\usepackage{bbm}
\usepackage{dsfont}

\else 

\usepackage{microtype}
\usepackage{graphicx}
\usepackage{subcaption}
\usepackage{booktabs} 

\usepackage{hyperref}

\usepackage{icml2026}

\usepackage{amsmath}
\usepackage{amssymb}
\usepackage{mathtools}
\usepackage{amsthm}

\fi

\ifdefined\isarxiv

\else
\usepackage{algpseudocode}
\fi
 
\allowdisplaybreaks

\ifdefined\isarxiv

\usepackage{tikz}
\usepackage{hyperref}  
\hypersetup{colorlinks=true,citecolor=blue,linkcolor=blue} 
\usetikzlibrary{arrows}
\usepackage[margin=1in]{geometry}

\else
\fi
 
\graphicspath{{./figs/}}

\theoremstyle{plain}
\newtheorem{theorem}{Theorem}[section]
\newtheorem{lemma}[theorem]{Lemma}
\newtheorem{definition}[theorem]{Definition}

\newtheorem{corollary}[theorem]{Corollary}

\newcommand{\wt}{\widetilde}
\newcommand{\ov}{\overline}
\newcommand{\N}{\mathcal{N}}
\newcommand{\R}{\mathbb{R}}

\renewcommand{\tilde}{\wt}

\newcommand{\Tmat}{{\cal T}_{\mathrm{mat}}}
\newcommand{\opt}{\mathsf{opt}}
\newcommand{\alg}{\mathsf{alg}}

\DeclareMathOperator*{\E}{{\mathbb{E}}}

\DeclareMathOperator{\OPT}{OPT}

\DeclareMathOperator{\poly}{poly}

\DeclareMathOperator{\rank}{rank}

\DeclareMathOperator{\tr}{tr}

\ifdefined\isarxiv

\else
\icmltitlerunning{\paperRTitle}

\fi

\begin{document}

\ifdefined\isarxiv

\date{}
\title{\paperTitle}
\author{\paperAuthor}

\else

\twocolumn[
  \icmltitle{\paperTitle}


  \icmlsetsymbol{equal}{*}

  \begin{icmlauthorlist}
    \icmlauthor{Firstname1 Lastname1}{equal,yyy}
    \icmlauthor{Firstname2 Lastname2}{equal,yyy,comp}
    \icmlauthor{Firstname3 Lastname3}{comp}
    \icmlauthor{Firstname4 Lastname4}{sch}
    \icmlauthor{Firstname5 Lastname5}{yyy}
    \icmlauthor{Firstname6 Lastname6}{sch,yyy,comp}
    \icmlauthor{Firstname7 Lastname7}{comp}
    \icmlauthor{Firstname8 Lastname8}{sch}
    \icmlauthor{Firstname8 Lastname8}{yyy,comp}
  \end{icmlauthorlist}

  \icmlaffiliation{yyy}{Department of XXX, University of YYY, Location, Country}
  \icmlaffiliation{comp}{Company Name, Location, Country}
  \icmlaffiliation{sch}{School of ZZZ, Institute of WWW, Location, Country}

  \icmlcorrespondingauthor{Firstname1 Lastname1}{first1.last1@xxx.edu}
  \icmlcorrespondingauthor{Firstname2 Lastname2}{first2.last2@www.uk}

  \icmlkeywords{Machine Learning, ICML}

  \vskip 0.3in
]

\printAffiliationsAndNotice{} 

\fi

\ifdefined\isarxiv
\begin{titlepage}
  \maketitle
  \begin{abstract}

Submodular function maximization is a critical building block for diverse tasks, such as document summarization, sensor placement, and image segmentation. Yet its practical utility is often limit by the $O(knd^2)$ computational bottleneck. In this paper, we propose an integrated framework that addresses efficiency and privacy simultaneously. First, we introduce a novel matrix-based computation paradigm that accelerates function evaluations. Second, we develop approximate data structures that further streamline the optimization process, achieving a theoretical complexity of $O(\epsilon^{-2}(nd+kn+kd^2)\log(k/\delta))$. Third, we integrate ($\epsilon, \delta$)-DP guaranties to address the privacy concerns inherent in sensitive optimization tasks.

  \end{abstract}
  \thispagestyle{empty}
\end{titlepage}

\newpage

\else

\begin{abstract}

\end{abstract}

\fi


\section{Introduction} \label{sec:introduction}
Submodular functions are one important set of functions in machine learning tasks.
A submodular function is a set function defined on a set of data points, and the function has the property that the incremental value for any individual data point to the function value decreases as the size of the data set increases.
Submodular functions have many real-world applications, including sparse estimation \cite{wib15}, mean field inference \cite{bian2019optimal}, variational inference \cite{djolonga2014map, chen2016submodular}, active learning \cite{wib15}, stochastic optimization \cite{golovin2011adaptive}, document summeriazation \cite{lin2011class}), sensor placement \cite{mehr2018submodular}, and image segmentation \cite{shen2019submodular}. 

In these applications, we typically face an optimization problem to find the maximum point of the submodular function subject to some constraints.
Even for the simplest constraints, e.g. cardinality constraints, maximize general submodular functions can be NP-hard \cite{feige1998threshold, krause2012near}.
Fortunately, it is well-known that submodular function maximization with cardinality constraints can be well approximated by greedily adding elements \cite{nwf78}.
Submodular function maximization with various other constraints and structure of objective functions, including matroid constraints \cite{calinescu2007maximizing, krause2009simultaneous, calinescu2011maximizing}, knapsack constraints \cite{leskovec2007cost, sviridenko2004note}, min-cost coverage function \cite{wolsey1982analysis}, and  mutual information function \cite{krause2012near, narasimhan2005q} are also studied. The first linear-time algorithm for maximizing a general monotone submodular function subject to a cardinality constraint was developed by~\cite{mbk+15}.
In practice, it sometimes happens that instead of having access to the complete input data, we do not know the information about the future and have to make online decisions \cite{streeter2008combining, munagala2005pipelined}.
Recently, online submodular function maximization has received much attention \cite{streeter2008online, chen2018online}.

In \cite{chen2018online}, continuous submodular function maximization in online setting was studied.
Projection-free algorithm was studied in  \cite{chen2018projection}.
Submodular bandit with extra structures, e.g. linear submodular bandit and adaptive submodular bandit problem, were studied \cite{yue2011linear, yu2016linear, gabillon2013adaptive}.

To accelerate submodular function maximization algorithms and save computation resources,
we build an embedding representation of the submodular maximization problem and propose a new method based on a novel search data structure which only spends linear in $n$ time at the initialization.
In each iteration of the algorithm, we only need to pay sublinear in $n$ time to pick an item to add into the set. The use of approximation algorithms is the key to enabling our data structures to perform efficient computations. After solve the classical problem, we further analyzed extensions under various constraints and the online setting. The results demonstrate that our data structures can also efficiently solve the extended problems.

In many real-world applications, the input data for submodular maximization involves sensitive information, such as social network connections, sensor locations, or user preferences. Directly applying existing maximization algorithms may leak such private information. To address this, differential privacy has been introduced as a framework for protecting individual data while still enabling useful computation.

In this work, we further extend our fast data structure based algorithm to the differentially private setting. Our method provides $(\epsilon, \delta)$-DP while preserving near optimal approximation guarantees, achieving efficient performance in both the offline and online scenarios.

{\bf Our Results}
The main result of this paper can be informally stated as Theorem \ref{thm:main_informal}.
For theorem \ref{thm:main_formal}, we put it in the introduction section to be a quick preview of the main result of this
work. So the detailed problem statement is placed in the section \ref{sec:preli}. The submodular
function is defined in Definition \ref{def:submodular_function} and the problem statement is the optimization problem described in
equation \eqref{eqn:def_submodular_maximization}. In Section~\ref{sec:dp_analysis}, we further extend our algorithm to the differential privacy setting and provide formal privacy guarantees.
\begin{theorem}[Main result, informal version of Theorem \ref{thm:main_formal}] \label{thm:main_informal}
There is a submodular algorithm such that, given a submodular function maximization problem with cardinality constraint and two precision parameters $\epsilon,\delta$ as input, it runs in $O(\epsilon^{-2}(nd+kn+kd^2)\log(k/\delta))$ time and returns a solution $S$ with $f(S) \ge (1-1/e)\max_{|S|=k}\{f(S)\} - k(2-1/e)\epsilon$  with probability at least $1-\delta$.
\end{theorem}



As a common baseline, the standard naive greedy algorithm has a time complexity of $O(n d^2 k)$. This algorithm is known to return a set $S$ satisfying $ f(S) \geq (1-1/e) \max_{|T|=k} \{f(T)\} $. 
See details in Theorem~\ref{thm:greedy} and Theorem~\ref{thm:trivial_algorithm}. The additive error's impact on achieving a constant factor approximation, including the role of the embedding dimension $d$, is analyzed in the Corollary~\ref{cor:error_approximate_greedy_algo}. 

{\bf Roadmap.}
The rest of our paper is organized as follows. 
Section \ref{sec:related_work} introduce the related work.
Section \ref{sec:preli} formally defines basic notations and our research problem. Section \ref{sec:tech_ow} provides an overview of our techniques. Section \ref{sec:benchmark} provides a benchmark of submodular algorithms, where all of our algorithms are based on this benchmark. Section~\ref{sec:previous_tools} provides a brief description about tools from previous works we use. Section~\ref{sec:ds} present our data structures, including auxiliary, main and extended data structures.
Section~\ref{sec:dp_analysis} provides the differential privacy analysis of our algorithms.
Section \ref{sec:conclusion} makes a summary of this paper.
\section{Related Work}\label{sec:related_work}
\paragraph{Submodular functions.}
The class of submodular functions \cite{nwf78} captures the phenomena of decreasing marginal gain.
In many machine learning applications, e,g., computer vision~\cite{jb11}, causal inference~\cite{sjs10,zs16}, compressed sensing~\cite{bach10,bjm+12}, categorical feature compression~\cite{bcef+2019}, fairness~\cite{cdkv16, bs15}, active learning~\cite{gb10, gk11, wib15}, dictionary learning~\cite{dk11}, data summarization~\cite{lb11, tiwb14}, crowd teaching~\cite{sbb+14}, variational inference~\cite{cck16}, and influence maximization~\cite{kkt03, gsdz+16,zbcb+16}, submodular functions are playing important roles.


\paragraph{Differential privacy.}
Differential Privacy (DP), first formalized in~\cite{dmns06}, has become widely recognized as the gold standard for rigorous privacy protection. A growing body of research has focused on adapting classical algorithms~\cite{aimn23, ll23, csw+23, syyz23_dp, ffl+25} and data structures~\cite{qjs+22} to satisfy differential privacy guarantees. More recently, the application of DP has expanded well beyond traditional settings, encompassing entire machine learning pipelines~\cite{lds+21, phk+23, samb24, qwh24, gsyz23_dp, lhr+24, hll+24}.

\paragraph{Sketching.}
In this work we use sketching technique to speedup an optimization problem which is called submodular optimization. Here we briefly a number of previous sketching work in different areas. Sketching has been applied to speedup many continuous optmization problems such as linear programming \cite{cls19,song19,b20,jswz21,sy21,gs22}, empirical risk minimization \cite{lsz19,qszz23}, cutting plane method \cite{jlsw20}, computing John Ellipsoid \cite{ccly19,syyz22}. 
Other than continuous optimization, several discrete optimization problems also can be applied \cite{dsw22,sxz22,z22,jlsz23}.  Sketching ideas also have applied to several theoretical large language models problems such as exponential and softmax regression \cite{lsz23,gsy23,dls23_softmax}, sketching the feature dimension of attention matrix \cite{dms23}.  In many machine learning tasks, sketching technique is also very useful, matrix completion \cite{gsyz23},  adversarial training \cite{gqsw22}, training over-parameterized neural tangent kernel regression \cite{bpsw21,szz21,z22,als+22,hswz22}, matrix sensing \cite{dls23_sensing,qsz23}, kernel density estimation \cite{qrs+22}, federated learning \cite{swyz23}. In addition, sketching technique also can be applied to theory of relational database \cite{qjs+22}.

\section{Preliminary} \label{sec:preli}

\subsection{Notations}

For a square matrix $A$, we use $\tr[A]$ to denote the trace of $A$. 
For a set $S\subseteq[n]$ , we use $\overline{S}$ to denote the complementary set of $S$, which is equivalent to $[n]\backslash S$.
If there is no ambiguity, we will omit $[n]$.
We use $n$ to denote the number of elements in the ground set $[n]$.
We use $d$ to denote the dimension of each vector.
For $t\ge 1$, we use $S_t$ to denote the set of selected elements at time $1,2,\cdots,t$.
For two $d$-dimensional vectors $u,v$, we use $\langle u,v \rangle$ to denote their inner product. For two $n\times m$-dimensional matrices $A,B$, we will also use $\langle A,B \rangle$ to denote their inner product, that is, their elementary-wise product. For any function $f$, we use $\wt{O}(f)$ to denote $\poly (\log f)$.

We formalize the submodular function maximization problem with cardinality constraint in this section.
Informally, a set function is submodular if it has decreasing marginal increment.
\begin{definition}[Submodular function]\label{def:submodular_function}
Let $f:2^{[n]}\rightarrow\mathbb{R}$ denote a set function.
Define the discrete derivative $\Delta_f$ as
$
   \Delta_f(i|S):=f(S\cup\{i\})-f(S). 
$
Function $f$ is submodular if for every $S\subseteq T$ and $i \in [n]-T$, it holds that
$
\Delta_f(i|T) \le \Delta_f(i|S).
$

\end{definition}
For the purpose of simplicity, we present the submodular maximization problem with cardinality constraint in this paper.
For other typical constraints like knapsack constraint and matroid constraint, our method could naturally be extended and we discuss the extension in section \ref{sec:extension}.

\subsection{Goal}
In this paper, we are going to propose algorithms solving optimization problem \eqref{eqn:def_submodular_maximization} efficiently.
\begin{align}\label{eqn:def_submodular_maximization} 
\max_{S\subseteq [n]}  f(S), \textrm{~~s.t.~~}   |S| \le k
\end{align}

\subsection{Representation of $f(S)$}
One problem arises in designing algorithm is how we select the representation of input instances.
As the constraint part in optimization problem \eqref{eqn:def_submodular_maximization} is simple, it remains to decide the representation of $f(S)$.
Suppose $S=\{i_1,i_2,\cdots,i_m\} \subseteq [n]$, we could always decompose $f(S)$ into sum of increment as
$
    f(S) = f(S_0) + \sum_{j=1}^m f(S_j) - f(S_{j-1}),
$
where $S_0 = \emptyset$ and $S_j = S_{j-1} + \{i_j\}$.
Without loss of generality, we could assume $f(\emptyset) = 0$.
Note that by the definition of $\Delta_f(i|S)$, we have
$
f(S_j) - f(S_{j-1}) = \Delta_f(i_j | S_{j-1}).
$
So the decomposition form can be simplified as
\begin{align}\label{eqn:decompose_f_2}
    f(S) = \sum_{j=1}^m \Delta_f(i_j | S_{j-1})
\end{align}



\begin{definition}[Quadratic embedding of $f$]
    For a family of set $\{u_i\} \in \mathbb{R}^{d}$ and a map $h : 2^{[n]} \rightarrow \mathbb{R}^{d \times d} $, $S \mapsto h(S)$,
    represent the submodular function $f$ if, for every $S \subseteq [n]$ and every
    $i \notin S$,
    \begin{align}\label{eqn:decompose_f_3}
        \Delta_f(i | S) = u_i^\top h(S) u_i
    \end{align}
\end{definition}
Thus our input instance $(n,k,f)$ of optimization problem \eqref{eqn:def_submodular_maximization} can be represented by $(n, k, U, h)$ where $U=[u_1,\cdots,u_n]$.
We emphasize that we are not assuming $f(S)$ must have some specific form here.
Note that for $U$ we have $n\cdot d$ free variables and for $h(S)$ we have $2^n \cdot d^2$ free variables.
Therefore, in total, we have $nd+2^nd^2$ free variables and there are $n\cdot 2^n$ different values of $\Delta_f(i|S)$.
Hence we have enough degrees of freedom to represent any submodular function $f$ in the form of \eqref{eqn:decompose_f_2} and \eqref{eqn:decompose_f_3} when $d \ge \sqrt{n}$. 

In practice, an important subclass of submodular functions is monotone submodular functions, i.e. functions $f$ satisfying $f(A) \le f(B)$ for all $A\subseteq B \subseteq [n]$.
When $f$ is monotone, we could restrict all $h(S)$ to be a positive semidefinite (PSD) matrix.
In the case of PSD $h(S)$, we can achieve faster acceleration.It should be stressed that using the quadratic embedding form allows the problem to be transformed into vector operations, thereby enabling the application of numerous matrix approximation algorithms to accelerate the computation.

\subsection{Locality Sensitive Hashing.}
\label{sec:previous_tools:lsh}
We present more previous lemmas on locality sensitive hashing.

The approximate maximum inner product search problem is to find a point $p$ that achieves the maximum inner product with respect to the given point up to a certain approximation error ratio.
\begin{definition}[Approximate Max-IP, \cite{ssx21}]
Let $c\in(0,1)$ and $\tau\in(0,1)$. Given an $n$-point dataset $P\subset \mathbb{S}^{d-1}$ on the sphere, the goal of the $(c,\tau)$-Maximum Inner Product Search (Max-IP) is to build a data structure that, given a query $q\in \mathbb{S}^{d-1}$ with the promise that there exists a datapoint $p\in P$ with $\langle p,q \rangle \ge \tau$, it reports a datapoint $p'\in P$ with similarity $\langle p',q \rangle \ge c\cdot \tau$.
\end{definition} 
The naive method for the approximate MAX-IP problem by enumeration could solve this problem in $O(nd)$ time and $O(nd)$ space.  
With more elaborate design, the recent work (\cite{ssx21}) gives better solution.
There are two choices provided in \cite{ssx21}.

\begin{theorem}[Theorem 8.2, page 19, \cite{ssx21}]\label{thm:max-IP}
Let $c\in(0,1)$ and $\tau\in(0,1)$. Given a set of $n$-points $Y\subset \mathbb{S}^{d-1}$ on the sphere, one can construct a data structure with ${\cal T}_{\mathrm{init}}$ preprocessing time and ${\cal S}_{\mathrm{space}}$ space so that for any query $x\in {\cal S}^{d-1}$, we take query time complexity $O(d\cdot n^{\rho}\log(1/\delta))$:
     if $\mathsf{Max}$-$\mathsf{IP}(x,Y) \ge \tau$, then we output a vector in $Y$ which is a $(c,\tau)-\mathsf{Max-IP}$ with respect to $(x,Y)$ with probability at least $1-\delta$, where $\rho:=f(c,\tau)+o(1)$.
     Otherwise, we output $\mathsf{fail}$.
Further, 
     If ${\cal T}_{\mathsf{init}}=O(dn^{1+\rho})\log(1/\delta)$ and ${\cal S}_{\mathsf{space}}=O(n^{1+\rho}+dn)\log(1/\delta)$, then $f(c,\tau)=\frac{1-\tau}{1-2c\tau+\tau}$.
     If ${\cal T}_{\mathsf{init}}=O(dn^{1+o(1)})\log(1/\delta)$ and ${\cal S}_{\mathsf{space}}=O(n^{1+o(1)}+dn)\log(1/\delta)$, then $f(c,\tau)=\frac{2(1-\tau)^2}{(1-c\tau)^2}-\frac{(1-\tau)^4}{(1-c\tau)^4}$.
\end{theorem}
The two results in above Theorem have trade-off. One algorithm depends on $n^{\rho}$ and the dependence between $\rho$ and $c$ is better. The other the algorithm depends on $n^{o(1)}$ which is faster than $n^{\rho}$, but the dependence between $\rho$ and $c$ is worse.

\subsection{Dynamic Distance Estimation} \label{sec:previous_tools:ADE}
\begin{definition}[Online Approximate Adaptive Distance Estimation]
Given $n$ vectors $x_1, \cdots, x_n$ from $d$-dimensional space, the goal is to design a data structure that supports query $z$, and it outputs a vector $u \in \R^n$ such that $u_i \approx \| x_i - z \|$ for all $i \in [n]$.
\end{definition}

Brute force: $O(nd)$ time, $O(nd)$ space.  
\begin{theorem}[Theorem 1.4 in \cite{cn22}]\label{thm:cn22}
There is a data structure for the Online Approximate Adaptive Distance Estimation Problem uses $\tilde{O}(\epsilon^{-2}nd\log(1/\delta))$ space with the following procedures:
     \textsc{Init}$(  \{x_1, x_2, \dots, x_n\}\subset \R^d, \epsilon \in (0,1), \delta \in (0,1))$: Given data points $\{x_1, x_2, \dots, x_n\}\subset \R^d$, an accuracy parameter $\epsilon$ and a failure probability $\delta$ as input, the data structure preprocesses in time $\tilde{O}(\epsilon^{-2}nd\log(1/\delta))$.
    \textsc{Query}$(q \in \R^d)$: Given a query point $q \in \R^d$, the \textsc{Query} operation takes $q$ as input and approximately estimates the Euclidean distances from $q$ to all the data points $\{x_1, x_2, \dots, x_n\}\subset \R^d$ in time $\tilde{O}(\epsilon^{-2}(n+d)\log(1/\delta))$ i.e. it provides a set of estimates $\{\tilde{d}_i\}_{i=1}^n$ such that:
     $  \forall i \in[n], (1-\epsilon)\|q-x_{i}\|_{2} \leq \tilde{d}_{i} \leq(1+\epsilon)\|q-x_{i}\|_{2} 
     $
     with probability at least $1 -\delta$, even for a sequence of adaptively chosen queries.
\end{theorem}

\subsection{Differential Privacy}
\begin{definition} [Differential Privacy, \cite{dr14}] \label{def:dp}
Let $\epsilon > 0$ and $\delta \geq 0$. A randomized function $\mathcal{A}$ is $(\epsilon, \delta)$-differentially private ($(\epsilon, \delta)$-DP) if for any two neighboring datasets $X \sim X'$, and any possible outcome of the algorithm $S \subset \mathrm{Range}(\mathcal{A})$, $\Pr[\mathcal{A} (X) \in S] \leq e^{\epsilon} \Pr[\mathcal{A} (X') \in S] + \delta$. 
\end{definition}

\begin{lemma} [Post-Processing Lemma for DP, \cite{dr14}] \label{lem:post_processing_dp}
Let $\mathcal{M} := \mathbb{N}^{|\chi|} \rightarrow \R $ be a randomized 
algorithm that satisfies $(\epsilon, \delta)$-DP. Let $f: \R \rightarrow \R'$ be an arbitrarily random mapping. Then $f \circ \mathcal{M}: \mathbb{N}^{|\chi|} \rightarrow \R'$
$(\epsilon, \delta)$-differentially private. 
\end{lemma}

\begin{lemma}[Basic composition \cite{dmns06}]\label{thm:basic_comp}
Given $t$ algorithms executed sequentially, where the $i$-th algorithm is $(\epsilon_i,\delta_i)$-DP for $\epsilon_i>0$ and $\delta_i \geq 0$, the overall mechanism obtained by composing them is $(\epsilon_1+\cdots+\epsilon_t,\delta_1+\cdots+\delta_t)$-DP.
\end{lemma}

\begin{lemma}[Composition lemma, \cite{drv10}]\label{lem:composition}
    Let $\epsilon \in (0, 1)$, and $M_1, \cdots, M_k$ be $\epsilon^\prime$-DP, adaptively chosen mechanisms, then the composition $M_1 \circ \cdots \circ M_k$ is ($\epsilon, \delta$)-DP, where
    $
        \epsilon^\prime = \frac{\epsilon}{\sqrt{8k\log (1 / \delta)}}.
    $
\end{lemma}

\begin{lemma}[Advanced Composition, Theorem 3.20 in page 53 of \cite{dr14}]
For all $\epsilon, \delta, \delta^\prime \geq 0$, the class of ($\epsilon, \delta$)-DP mechanisms satisfies ($\epsilon, k\delta + \delta^\prime$)-DP under $k$-fold adaptive composition for:
$
    \epsilon^\prime = \sqrt{2k\ln (1 / \delta^\prime)}\epsilon + k\epsilon (e^\epsilon - 1)
$
\end{lemma}

\begin{lemma}[Exponential Mechanism, \cite{mt07}]\label{lem:exponential_mechanism}
Let $q: {\cal D} \times {\cal R} \to \R$ be a utility function with sensitivity $\Delta q = \max_{r \in {\cal R}} \max_{D,D'} | q(D,r) - q(D',r) |$, where $D, D'$ are neighboring datasets.
The exponential mechanism ${\cal M}_E$ selects $r \in {\cal R}$ with probability $\Pr[{\cal M}_E(D) = r] \propto \exp (\frac{\epsilon \cdot q(D,r)}{2\Delta q})$. Then $\mathcal{M}_E$ satisfies $(\epsilon,0)$-differential privacy.
\end{lemma}
\section{Technique Overview} \label{sec:tech_ow}




Our data structure design starts with the greedy search approach \cite{nwf78}. The proposed algorithm in \cite{nwf78} is an approximate algorithm.  
Specifically, under the representation of $f(S)$ in Eq. \eqref{eqn:decompose_f_3}, \cite{nwf78} uses an algorithm which starts from empty set $S_0=\emptyset$, and in each time step $t$, it constructs $S_t$ by adding increase $S_{t-1}$ with a singleton set that maximizes the marginal gain of $f$.
Formally, \cite{nwf78} chooses 
$
    i = \arg\max_{i \in \overline{S}_{t-1} }\{u_i^\top h(S_{t-1})u_i\}
$
and makes $S_t = S_{t-1}+\{i\}$.
It is proved that this algorithm can achieve at least a constant times ($\approx 0.63$) approximation of the optimal algorithm.
Many following works solving submodular function maximization with other types of constraints and solving online submodular function maximization use similar greedy updating rules.
Hence our acceleration method could be applied to many algorithms developed later too.

{\bf Transforming quadratic form into inner product.}
Our first important technique of this paper is the following two observations: (1) for $d$-dimensional vector $u$ and $d \times d$ matrix $A$,
$
    u^\top A u = \langle u u^\top, A \rangle;
$
(2) after flattening two $d \times d$ matrices $A,B$ into $d^2$-dimensional vectors $v_A,v_B$, their inner product remain unchanged, that is, 
$
    \langle A,B \rangle = \langle v_A,v_B \rangle.
$
From this perspective, the process of maximizing $u_j^\top h(Q) u_j$ can be implemented by computing the inner product of two vectors, where one vector is the flattened $u_ju_j^\top$, hence only related to $u_j$, and the other vector is only related to $h(Q)$.

{\bf Data structures for submodular function maximization.} After  transforming quadratic form into inner product , we design a dynamic inner product estimation data structure,  IPE.  IPE uses an approximate algorithm to compute the inner products between all elements in a vector set $X:=\{x_i\in R^d\}$ and a given matrix $A$. Then, we designed another data structure, DS, which incorporates  IPE as one of its components. Within DS, the query procedure invokes  IPE and returns $\arg\max_{x_i\in X}x_i^\top A x_i$. Subsequently, by leveraging DS, we achieved acceleration for the greedy algorithm. 
\section{Benchmark: Greedy Selection}\label{sec:benchmark}
In this section, we present 4 versions of greedy selection algorithm. We also give analysis on their performance, because they are important for analysis of our own data structure. In Section~\ref{sec:benchmark:greedy}, we present original greedy algorithm and its approximate version. Then in Section~\ref{sec:benchmark:naive}, we present greedy algorithm with naive implementation, which means that the algorithm calculate submodular function value by quadratic form and inner product. In Section~\ref{sec:benchmark:batch}, we introduce the batch version of greedy algorithm.

\subsection{Analysis on Greedy Algorithms}\label{sec:benchmark:greedy}
The greedy selection algorithm starts with an empty set $S_0 = \emptyset$.
In each iteration, the algorithm chooses the element maximizing the marginal increment to add into the set, until the set eventually contains $k$ elements.
To be more specific, in iteration $t\in\{1,2,\cdots,k\}$, we let
$S_t \leftarrow S_{t-1} + \{j_t\}.$
where the element in the singleton is given by
$
    j_t=\arg\max_{j\in\ov{S}_{t-1}} f(S_{t-1}+\{j\})
$.
Although the greedy selection algorithm is very simple, it is effective in the sense that the approximation error of it is $1-1/e$. 

\begin{theorem}[\cite{nwf78}] \label{thm:greedy}

For a monotone submodular function $f$, the greedy algorithm (see Algorithm \ref{alg:greedy}) guarantees to output a set $S$ satisfying 
$
   f(S) \ge (1-1/e)\max_{|T|=k}\{f(T)\} 
$.
\end{theorem}

\begin{corollary}[Robustness of approximate greedy algorithm] \label{cor:error_approximate_greedy_algo}
Given $\epsilon > 0$. Consider an approximate greedy algorithm $A$ that has an oracle $O$: taken an arbitrary set $S\subseteq[n]$ and $i \in [n]\backslash S$, returns a value $O(S,i)$ with a guarantee that 
$
    \Delta(i|S)-\epsilon \le O(S,i) \le \Delta(i|S)+\epsilon
$.
Additionally, in each time step $t = 1,2,\cdots,k$, it selects 
 $
    j_t = \arg\max_{j}\{O(S_{t-1},j)\}
 $
and lets $S_t \gets S_{t-1} \cup \{j_t\}$, then this algorithm makes sure that it returns a set $S_k$ with
 $
f(S_k) \ge (1-1/e) \max_{|T|=k}\{f(T)\} - k(2-1/e)\epsilon.
 $
\end{corollary}
\begin{proof}
Define $f':2^{[n]}\to \R$ as $f'(S):=f(S-\{i\})+O(S-\{i\},i)$. 
Consider the real greedy algorithm $A^*$ which really selects $j_t = \arg\max_{j} \{\Delta(j|S_{t-1})\}$ and adds $j_t$ to $S_{t-1}$ every time. Define another set function $f':2^{[n]}\to \R$ by $\Delta f'(i|S) = O(S,i)$. Then $A$ is a greedy algorithm acts on submodular function $f$. For the two submodular function maximization problem with submodular function $f$ and $f'$, suppose algorithm $A$ outputs sets $S_A$ with $f(S_A) = \alg$ and $f'(S_A) = \alg'$ respectively, and suppose algorithm $A^*$ outputs sets $S_{A^*}$ with $f(S_{A^*}) = \opt$ and $f'(S_{A^*}) = \opt'$ respectively. 

By Theorem \ref{thm:greedy}, since $A$ is the greedy algorithm of the submodular maximization problem with $f'$, $\alg'$ is at least $1-1/e$ times the optimal algorithm of the submodular maximization problem with $f'$, hence is at least $1-1/e$ times the output value of $A^*$, which is, $\alg' \ge (1-1/e) \opt'$. Since $O(S,i) \ge \Delta(i|S) - \epsilon$, it holds $\opt' = f'(S_{A^*}) \ge f(S_{A^*}) - k\epsilon = \opt - k\epsilon$. Since $O(S,i) \le \Delta(i|S) + \epsilon$, it holds $\alg' = f'(S_A) \le f(S_A) + k\epsilon = \alg + k \epsilon$. Combining the 3 equations, we have 
$
\alg \ge \alg' - k\epsilon \ge (1-1/e)\opt' - k\epsilon \ge (1-1/e)\opt - k(2-1/e)\epsilon.
$

\end{proof}

\subsection{Greedy Selection with Naive Implementation}\label{sec:benchmark:naive}

 The greedy selection algorithm can be implemented naively  via enumerating elements in $\ov{S}_{t-1}$.

\begin{theorem} \label{thm:trivial_algorithm}
The naive implementation of the greedy selection algorithm (Algorithm \ref{alg:trivial}) uses $O(d^2)$ space, takes $O(1)$ time to initialize, and uses $O(nd^2)$ time to add a new element to the set.
\end{theorem}

\begin{proof}
Algorithm \ref{alg:trivial} correctly implements the greedy selection.
In each iteration of the for-loop beginning at line \ref{line:enumeration} in Algorithm \ref{alg:trivial}, since $A_t$ is a $d\times d$ matrix, computing $w_j$ takes $O(d^2)$ time.
As $|\ov{S}_t |\le n$, adding each element takes $O(nd^2)$ time.
\end{proof}

\subsection{Greedy Algorithm with Batch}\label{sec:benchmark:batch}

We could implement the greedy selection algorithm in a more efficient way by considering the candidates of the new element added in a batch-by-batch way.

\begin{theorem}\label{thm:trivial_algorithm_v2_informal}
There is an implementation of the greedy selection algorithm (Algorithm \ref{alg:trivial_v2}) that uses $O(d^2)$ space, takes $O(1)$ time during the initialization, spends $O(n \cdot d^{\omega-1})$ time in each iteration (to add an element to the set), where $\omega$ is the constant for fast matrix multiplication in theorem \ref{thm:fast_matrix_multiplication}.
\end{theorem}
\begin{proof}
First, we prove Algorithm \ref{alg:trivial_v2} is correct.
Since $U = [~u_{i_1}~u_{i_2}~\cdots~u_{i_d}~]$, we have 
$
    (U^\top A_t U)_{jj} = u_{i_j}^\top A_t u_{i_j}.
$.
Thus for each index $j\in  \ov{S}_t$, $u_j^\top A_t u_j$ will be compared with the current maximal value $w_{\max}$, hence 
$
    j_{\max}=\arg\max_{j\in \ov{S}_t}\{ u_j^\top A_t u_j \}
$.
Then we prove the running time of Algorithm \ref{alg:trivial_v2}.
The outer for-loop beginning at line \ref{line:clever_enumeration} runs for $m = \lceil | \ov{S}_t|/d \rceil$ iterations. In each iteration, computing $M$ takes $O(d^\omega)$ time, and updating $w_{\max}$ and $j_{\max}$ takes $O(d)$ time.
Hence adding an element takes $O(md^\omega)=O(nd^{\omega-1})$ time. 
\end{proof}

\section{Tools from Previous Works}\label{sec:previous_tools}

In this work we use sketching technique to speedup submodular optimization. Here we briefly a number of previous sketching work in different areas. Sketching has been applied to speedup many continuous optmization problems such as linear programming \cite{cls19,song19,b20,jswz21,sy21,gs22}, empirical risk minimization \cite{lsz19,qszz23}, cutting plane method \cite{jlsw20}, computing John Ellipsoid \cite{ccly19,syyz22}. 
Other than continuous optimization, several discrete optimization problems also can be applied \cite{dsw22,sxz22,z22,jlsz23}.  Sketching ideas also have applied to several theoretical large language models problems such as exponential and softmax regression \cite{lsz23,gsy23,dls23_softmax}, sketching the feature dimension of attention matrix \cite{dms23}. 

We briefly introduce these tools in this section, more details are in Section~\ref{sec:app_preli}. In our algorithm, we need to operate dense matrix multiplication. Currently, the fastest matrix multiplication achieves $O(n^{2.373})$ time complexity(\cite{aw21}).The approximate maximum inner product search problem(MAX-IP) is to find a point p that achieves the maximum inner product with respect to the given point up to a certain approximation error ratio. \cite{ssx21} give us tools to solve MAX-IP efficiently, and we will use that for our data structure. In our data structure  IPE, we use tools from~\cite{cn22} to query distance between a vector $q$ and all vectors in a given data set.

\section{Data Structures and Main Algorithm}\label{sec:ds}
In this section, we introduce several data structures designed to solve fast submodular function maximization problem. We also present an algorithm that demonstrates how to utilize these data structures to solve the problem. Additionally, we propose an extension using LSH of a data structure. The performance analysis of both the data structures and the algorithm is provided in this section. Since approximation algorithms require specific parameters to control accuracy, our data structure need additional hyperparameters. Detailed implementations are included in Section~\ref{sec:app_data_structure} and Section~\ref{sec:app_alg_lsh}. To be more specific, Section~\ref{sec:ds:IPE} present IPE data structure to estimate inner product. Then Section~\ref{sec:ds:fast_search_ds} shows fast quadratic form search data structure DS which use IPE. Then in Section~\ref{sec:ds:problem_with_ds}, we give a algorithm which utilize DS to solve fast submodular function maximization problem. Finally in Section~\ref{sec:ds:fast_ds_lsh}, we propose a extension of DS. The description of main operations in our data structures are in the statements of our theorems. 
\subsection{Dynamic Inner Product Estimation} \label{sec:ds:IPE}
Now we describe our data structure design for inner product estimation, IPE. First, we give a definition describe IPE's main function. 
\begin{definition}[Online Approximate Adaptive Inner Product Estimation]
Given a set $\{x_1,x_2,\cdots,x_n\}$ of $d$-dimensional vectors, the goal of online approximate adaptive inner product estimation is to design a data structure with the  procedures \textsc{Update} and \textsc{Query}:
     \textsc{Update}$(i\in[n],z\in\R^d)$. It supports to replace $x_i$ by $z$.
    \textsc{Query}$(q\in\R^d)$. Given a query vector $q$, it returns a set $\{\wt{w}_i\}_{i=1}^n$ with guarantee that for any $i\in [n]$, $\wt{w}_i$ is close to $\langle q,x_i \rangle$ with a high joint probability. How precise this needs to be can be determined by practical demands.
\end{definition}

Next, we present a theorem describing procedures and performance of data structure IPE. \textsc{Query} operation is the most important function of IPE. 

\begin{theorem}[Procedures and performance of data structure IPE]\label{thm:IPE}
There is a data structure uses $\tilde{O}(\epsilon^{-2}D^2nd\log(1/\delta))$ space for the Online Approximate Adaptive Inner Product Estimation Problem with the following procedures:
    \textsc{Init}$(  \{x_1, x_2, \dots, x_n\}\subset \R^d, \epsilon \in (0,1), \delta \in (0,1))$: Given data points $\{x_1, x_2, \dots, x_n\}\subset \R^d$ ($\| x_i \| \leq D$, for all $i \in [n]$), an accuracy parameter $\epsilon$ and a failure probability $\delta$ as input, the data structure preprocesses in time $\wt{O}(\epsilon^{-2} D^2 n d \log(1/\delta))$.
    \textsc{Update}$(i \in [n], z \in \R^d)$: Given index $i$ and coordinate $z$, the data structure replace $x_i$ by $z$ in time $\tilde{O}(\epsilon^{-2}D^2d\log(1/\delta))$.
    \textsc{Query}$(q \in \R^d)$: Given a query point $q \in \R^d$ (where $\| q \| \leq 1$), the \textsc{Query} operation takes $q$ as input and approximately estimates the inner product of $q$ and all the data points $\{x_1, x_2, \dots, x_n\}\subset \R^d$ in time $\tilde{O}(\epsilon^{-2}D^2(n+d)\log(1/\delta))$ i.e. it provides a set of estimates $\{\tilde{w}_i\}_{i=1}^n$ such that:
    $
       \forall i \in[n], \langle q,x_{i}\rangle - \epsilon \leq \tilde{w}_{i} \leq \langle q,x_{i}\rangle + \epsilon
    $
     with probability at least $1 -\delta$, even for a sequence of adaptively chosen queries.
\end{theorem}
The proof is in Section~\ref{sec:app_data_structure}.
\subsection{Fast Quadratic Form Search Data structure 
} \label{sec:ds:fast_search_ds}



In this section, we propose an improved quadratic form search data structure. The \textsc{Query} operation is the most critical function, and its implementation relies on the \textsc{Query} within the IPE data structure.

\begin{theorem}[Data structure part of Theorem~\ref{thm:main_informal}. Restatement in Theorem~\ref{thm:app_our_ds2}]\label{thm:our_ds2}
There exists a data structure uses $\tilde{O}(\epsilon^{-2}D^2nd^2\log(1/\delta))$ space with the following procedures:

 \textsc{Init}{$(\{u_1,u_2,\cdots,u_n\}\subset \mathbb{R}^d, \epsilon\in(0,1), \delta\in(0,1))$}. Given vectors $u_1,u_2,\cdots,u_n$ satisfying $\|u_i\|_2 \le D$ for any $i\in [n]$, accuracy parameter $\epsilon$ and failure probability $\delta$, the data structure initializes in time $\tilde{O}(\epsilon^{-2}D^2nd^2\log(1/\delta))$.
    \textsc{Query}{$(M\in \mathbb{R}^{d\times d})$}. 
    the data structure outputs a $j_0\in\ov{S}_t$ 
    such that $u_{j_0}^\top M u_{j_0} \ge \max_{j}\{ u_j^\top M u_j \} - 2\epsilon$ with probability at least $1-\delta$ in time $\tilde{O}(\epsilon^{-2}D^2(n+d^2)\log(1/\delta))$. 
     \textsc{Update}$(i \in [n], z \in \R^d)$: Given index $i$ and coordinate $z$, the data structure replace $x_i$ by $z$ in time $\tilde{O}(\epsilon^{-2}D^2d\log(1/\delta))$.
    \textsc{Delete}{$(i\in[n])$}. 
    The data structure deletes $u_i$ in time $O(\log n)$.

\end{theorem}

The proof is in Section~\ref{sec:app_data_structure}.

\subsection{Submodular Function Maximization with Data Structure}\label{sec:ds:problem_with_ds}

In this section, we will demonstrate how to leverage the already designed data structures to solve the problem. First, we give a formal algorithm to describe the procedure.

Algorithm~\ref{alg:fast_submodular} shows  how to use our data structure to solve fast submodular function maximization problem. It is our main algorithm.
\begin{algorithm}[!ht]\caption{Fast Submodular Function Maximization Algorithm,informal version of Algorithm~\ref{alg:ours}. }\label{alg:fast_submodular}
\begin{algorithmic}[1]
\Procedure{FastSubmodularAlgorithm}{$n,k,\{u_1, \cdots, u_n \} \subseteq \R^d$} \Comment{Theorem~\ref{thm:main_formal}}
    \State $S_0 \gets \emptyset$
    \State DS ds.\textsc{Init} ($u_1,u_2,\cdots,u_n,\epsilon,\delta/k$) \Comment{Algorithm~\ref{alg:ds_v2}}
    \For{$t=0 \to k-1$}
        \State $A_t \gets h (S_t)$, $j \gets $ ds.\textsc{Query}($A_t$), $S_{t+1} \gets S_t \cup \{j\}$
        \State ds.\textsc{Delete}($j$)
    \EndFor
    \State \Return $S_k$
\EndProcedure
\end{algorithmic}
\end{algorithm}

Then,we present a theorem, which illustrate the performance of Algorithm~\ref{alg:ours}.

\begin{theorem}[Combination of our data structure with the submodular function maximization problem. Formal version of Theorem~\ref{thm:main_informal}. ] \label{thm:main_formal}
There is a submodular algorithm such that, given a submodular function maximization problem with cardinality constraint and two precision parameters $\epsilon,\delta$ as input, it runs in $O(\epsilon^{-2}(nd+kn+kd^2)\log(k/\delta))$ time and returns a solution $S$ with $f(S) \ge (1-1/e)\max_{|S|=k}\{f(S)\} - k(2-1/e)\epsilon$  with probability at least $1-\delta$.
\end{theorem}
The proof is in Section~\ref{sec:combination}.

\subsection{An Extension: Fast Quadratic Form Search Data Structure with LSH}\label{sec:ds:fast_ds_lsh} 

Recent progress in locality sensitive hashing would also help us design faster algorithms.
Informally, we call a family of hash functions sensitive when it tends to project points closed to each other to the same value and tends to project points far away from each other to different value.
The formal definition of sensitive hashing is given in definition \ref{def:locality_sensitive_hashing}.

In this section, we give a greedy selection algorithm with data structures related to locality sensitive hashing (LSH). First, we present a theorem describing the data structure. 

\begin{theorem}[Submodular function maximization problem with LSH] \label{thm:main_formal_LSH}
There exists a data structure (see Algorithm \ref{alg:app_lsh_fqfs}) with the following procedures:
\textsc{Init}{$(\{u_1,u_2,\cdots,u_n\}\subset \mathbb{R}^d, c\in(0,1), \delta\in(0,1))$}. Given vectors $u_1,u_2,\cdots,u_n$ satisfying $\|u_i\|_2 \le 1$ for any $i\in [n]$, accuracy parameter $c \in (0,1)$, threshold parameter $\tau \in (0,1)$ and failure probability $\delta \in (0,1)$, the data structure initializes with time $O(d^2n^{1+\rho}\log(1/\delta))$.
\textsc{Query}{$(M\in \mathbb{R}^{d\times d})$}. Given a $d\times d$ matrix $M$ with $\|M\|_F \le 1$, as long as there exists a $j \in [n]$ with $u_j^\top M u_j \ge \tau$, then the data structure outputs an undeleted index $j_0\in [n]$ 
    such that $u_{j_0}^\top M u_{j_0} \ge c \max_{j}\{ u_j^\top M u_j \}$ with probability at least $1-\delta$ in time $O(d^2n^{\rho}\log(1/\delta))$. 
\textsc{Delete}{$(i\in[n])$}. Given index $i$, the data structure deletes $u_i$ from the candidate index set in time $O(n^{\rho}\log(1/\delta))$. Here 
 $
\rho = \frac{1-\tau}{1-2c\tau+\tau} + o(1)
 $.
\end{theorem}
\begin{proof}
For procedure \textsc{Init}, computing $v_1,\cdots,v_n$ needs $O(nd^2)$ time, and by Lemma \ref{thm:max-IP}, since the dimension of every $Q(v_i)$ is $d^2+2 = O(d^2)$, maxip.\textsc{Init}$(Q(v_1),Q(v_2),\cdots,Q(v_n),\epsilon,\delta)$ runs in $O(d^2 n^{1+\rho}\log(1/\delta))$ time. Thus procedure \textsc{Init} runs in $O(d^2n^{1+\rho}\log(1/\delta))$  time.

For procedure \textsc{Query}, computing $m$ needs $O(d^2)$ time, and by Lemma \ref{thm:max-IP}, since $P(m)$ also has dimension $d^2+2 = O(d^2)$, maxip.\textsc{Query}$(P(m))$ needs in $O(d^2n^{\rho}\log(1/\delta))$ time. Thus procedure \textsc{Query} runs in $O(d^2n^{\rho}\log(1/\delta))$ time.

For procedure \textsc{Delete}, deleting $i$ from $S$ needs $O(1)$ time, and by Lemma \ref{thm:max-IP}, deleting $i$ from maxip needs $O(n^{\rho} \log(1/\delta))$ time. Thus procedure \textsc{Delete} runs in $O(n^{\rho}\log(1/\delta))$ time.

For the correctness, first note that for any $i \in [n]$, $\langle P(m), Q(v_i) \rangle = \langle m, v_i \rangle = \langle M, u_iu_i^\top \rangle = u_i^\top M u_i$. And since the norm of $P(m)$ and $Q(v_i)$ are both 1, by Lemma \ref{thm:max-IP}, as long as there exists $j$ such that $u_j^\top M u_j \ge \tau$,  maxip.\textsc{Query}$(P(m))$ outputs an index $i$ with $u_i^\top M u_i \ge c \max_{j\in [n]}\{u_j^\top M u_j\}$ with probability at least $1-\delta$.

Hence we complete the proof.
\end{proof}


Similar to data structure DS, applying this data structure to the fast submodular function maximization yields the following corollary.
\begin{corollary}[Combination of this data structure with the submodular function maximization problem]\label{cor:ds_LSH_submodular_problem}
Suppose a nondecreasing submodular function $f$ satisfies that for any set $S \subseteq [n]$ and an element $i \in [n] \backslash S$, $f(S \cup \{i\}) - f(S)$ (defined as $\Delta(i|S)$) can be expressed as $u_i^\top h(S) u_i$, where $u_i$ is a vector only related to $i$ with $\|u_i\| \le 1$, and $h$ is a function only related to $S$ with $\|h(S)\|_F \le 1$. 

Then there is a submodular algorithm (see Algorithm \ref{alg:app_lsh_combination_alg}) such that, given a submodular function maximization problem with cardinality constraint and precision parameter $\epsilon$, threshold parameter $\tau$ and failure tolerance $\delta$ as input, it runs in $O(d^2n^{\rho} (n+k) \log(k/\delta))$ time (where $\rho = \frac{1-\tau}{1-2c\tau+\tau} + o(1)$) and returns a solution $S$ with $f(S) \ge c(1-1/e)\max_{|S|=k}\{f(S)\}$  with probability at least $1-\delta$.
\end{corollary}
The proof is in Section~\ref{sec:app_alg_lsh}.
\section{Differential Privacy Analysis}\label{sec:dp_analysis}
In this section, we extend our algorithm to operate under differential privacy and analyze its privacy and utility guarantees.

Recall that in the non-private algorithm, at each iteration $t$ we select $j_t = \arg\max_{j \in \ov{S}_t} u_j^\top A_t u_j$, where $A_t = h(S_t)$. To achieve differential privacy, we replace this deterministic selection with a randomized selection according to the exponential mechanism (Lemma~\ref{lem:exponential_mechanism}):
\begin{align*}
    \Pr[j_t = j] \propto \exp (\frac{\epsilon \cdot \wt{w}_j}{2 \Delta}), \quad j \in S_t,
\end{align*}
where $\wt{w}_j$ is the estimated inner product from the dynamic data structure (Algorithm~\ref{alg:ds_v2}), and $\Delta$ is the sensitivity of the score function $f_j(S_t) = u_j^\top A_t u_j$.

\begin{lemma}[Differential privacy of the private selection algorithm]
Let $k \in \N$ be the number of iterations of the submodular maximization algorithm. Let $\epsilon' > 0$ be the privacy parameter used in the exponential mechanism at each iteration. Let $f_j(S_t)$ be the score function used in iteration $t$. 
Then, running Algorithm~\ref{alg:ours} with the exponential mechanism applied to the score $\wt{w}_j \approx f_j(S_t)$ at each iteration ensures that the entire algorithm satisfies ($\epsilon, \delta$)-DP under composition, with $\epsilon = \epsilon' \sqrt{8k\log (1 / \delta)}.$
\end{lemma}
\begin{proof}
At each iteration $t$, selecting $j_t$ according to the exponential mechanism ensures $\epsilon'$-differential privacy for that single selection.

By the composition lemma (Lemma~\ref{lem:composition}), running $k$ such iterations sequentially implies that the overall algorithm satisfies ($\epsilon, \delta$)-DP, where $\epsilon = \epsilon' \sqrt{8k\log (1 / \delta)}$.
\end{proof}

\begin{lemma}[Utility analysis of the private selection algorithm]
Let $F: 2^{[n]} \to \R_{\geq 0}$ be a monotone submodular function, and let $S^\mathrm{DP}_k$ be the output of Algorithm~\ref{alg:ours} when running $k$ iterations with the exponential mechanism, using estimated scores $\wt{w}_j \approx f_j(S_t)$ with sensitivity $\Delta$. Let $\opt$ be the optimal value of $F$, and let $\eta$ be an upper bound on the per-iteration estimation error from the dynamic inner product estimation (Algorithm~\ref{alg:inner_product_estimation}).

Then, for any failure probability $\beta \in (0, 1)$, with probability at least $1 - \beta$, $F(S_k^\mathrm{DP}) \ge (1-1/e) \cdot OPT - \frac{2 k \Delta}{\epsilon'} \log \frac{n}{\beta} - k \eta$, where the second term represents the utility loss due to the randomness of the exponential mechanism, and the third term accounts for the approximation error from the dynamic inner product estimation.
\end{lemma}
\begin{proof}
At each iteration $t$, the exponential mechanism selects $j_t$ with probability $Pr[j_t = j] \propto \exp (\frac{\epsilon \cdot \wt{w}_j}{2 \Delta}), \quad j \in S_t$, where $\wt{w}_j$ is the estimated score. By the utility guarantee of the exponential mechanism, with probability at least $1 - \beta / k$, $\wt{w}_{j_t} \geq \max_{j \in S_t} \wt{w}_j - \frac{2 \Delta}{\epsilon'} \log \frac{|S_t| k}{\beta}$.

Accounting for the estimation error $\eta$ from the dynamic inner product estimation, we have $f_{j_t}(S_t) \geq \wt{w}_{j_t} - \eta \geq \max_{j \in S_t} f_j(S_t) - \frac{2 \Delta}{\epsilon'} \log \frac{n k}{\beta} - \eta$. Applying this bound over all $k$ iterations of the greedy algorithm, and using the standard ($1 - 1 / e$) approximation guarantee for monotone submodular functions, we obtain $F(S_k^\mathrm{DP}) \geq (1-1/e) \cdot \opt - \frac{2 k \Delta}{\epsilon'} \log \frac{n}{\beta} - k \eta$, with probability at least $1 - \beta$ by a union bound over all iterations.
\end{proof}
\section{Conclusion}\label{sec:conclusion}
In this paper, our focus revolves around delving into submodular function maximization problem.
As proved in 
previous works, greedy selection algorithm and its variants achieve satisfying approximation error in many settings. 
Consequently, we are inspired to embark on a quest for innovative approaches that can accelerate the greedy selection algorithm and improve the efficiency.
We first develop an embedding representation of the input instance of submodular maximization problems.
Based on the embedding representation, we design our tree-based dynamic data structure so that each greedy selection step can be viewed as a query to our data structure.
By designing the data structure carefully, after greedy selection in each iteration, our data structure can be correspondingly updated efficiently, which is ready for the selection in next iteration.
With operations implemented efficiently, our method achieves acceleration for greedy algorithm and its variants for submodular maximization problems.
Furthermore, we extend our approach to the differential privacy setting, where we apply the exponential mechanism to each selection step, providing formal privacy guarantees while preserving high-quality solutions.


\section*{Impact Statement}

This paper presents work whose goal is to advance the field of Machine Learning. There are many potential societal consequences of our work, none of which we feel must be specifically highlighted here.

\ifdefined\isarxiv
\bibliographystyle{alpha}
\bibliography{ref}
\else
\bibliographystyle{icml2026}
\bibliography{ref}

\newcommand{\etalchar}[1]{$^{#1}$}
\begin{thebibliography}{GRSD{\etalchar{+}}16}

\bibitem[AIMN23]{aimn23}
Alexandr Andoni, Piotr Indyk, Sepideh Mahabadi, and Shyam Narayanan.
\newblock Differentially private approximate near neighbor counting in high
  dimensions.
\newblock {\em Advances in Neural Information Processing Systems},
  36:43544--43562, 2023.

\bibitem[ALS{\etalchar{+}}22]{als+22}
Josh Alman, Jiehao Liang, Zhao Song, Ruizhe Zhang, and Danyang Zhuo.
\newblock Bypass exponential time preprocessing: Fast neural network training
  via weight-data correlation preprocessing.
\newblock {\em arXiv preprint arXiv:2211.14227}, 2022.

\bibitem[AR15]{ar15}
Alexandr Andoni and Ilya Razenshteyn.
\newblock Optimal data-dependent hashing for approximate near neighbors.
\newblock In {\em Proceedings of the forty-seventh annual ACM symposium on
  Theory of computing}, pages 793--801, 2015.

\bibitem[AW21]{aw21}
Josh Alman and Virginia~Vassilevska Williams.
\newblock A refined laser method and faster matrix multiplication.
\newblock In {\em Proceedings of the 2021 ACM-SIAM Symposium on Discrete
  Algorithms (SODA)}, pages 522--539. SIAM, 2021.

\bibitem[Bac10]{bach10}
Francis Bach.
\newblock Structured sparsity-inducing norms through submodular functions.
\newblock {\em Advances in Neural Information Processing Systems}, 23, 2010.

\bibitem[BBK19]{bian2019optimal}
Yatao Bian, Joachim Buhmann, and Andreas Krause.
\newblock Optimal continuous dr-submodular maximization and applications to
  provable mean field inference.
\newblock In {\em International Conference on Machine Learning}, pages
  644--653. PMLR, 2019.

\bibitem[BCE{\etalchar{+}}19]{bcef+2019}
MohammadHossein Bateni, Lin Chen, Hossein Esfandiari, Thomas Fu, Vahab
  Mirrokni, and Afshin Rostamizadeh.
\newblock Categorical feature compression via submodular optimization.
\newblock In {\em International Conference on Machine Learning}, pages
  515--523. PMLR, 2019.

\bibitem[BJM{\etalchar{+}}12]{bjm+12}
Francis Bach, Rodolphe Jenatton, Julien Mairal, Guillaume Obozinski, et~al.
\newblock Optimization with sparsity-inducing penalties.
\newblock {\em Foundations and Trends{\textregistered} in Machine Learning},
  4(1):1--106, 2012.

\bibitem[BPSW21]{bpsw21}
Jan van~den Brand, Binghui Peng, Zhao Song, and Omri Weinstein.
\newblock Training (overparametrized) neural networks in near-linear time.
\newblock In {\em ITCS}, 2021.

\bibitem[Bra20]{b20}
Jan van~den Brand.
\newblock A deterministic linear program solver in current matrix
  multiplication time.
\newblock In {\em Proceedings of the Fourteenth Annual ACM-SIAM Symposium on
  Discrete Algorithms (SODA)}, pages 259--278. SIAM, 2020.

\bibitem[BS15]{bs15}
Eric Balkanski and Yaron Singer.
\newblock Mechanisms for fair attribution.
\newblock In {\em Proceedings of the Sixteenth ACM Conference on Economics and
  Computation}, pages 529--546, 2015.

\bibitem[CCK16a]{chen2016submodular}
Lin Chen, Forrest~W Crawford, and Amin Karbasi.
\newblock Submodular variational inference for network reconstruction.
\newblock {\em arXiv preprint arXiv:1603.08616}, 2016.

\bibitem[CCK16b]{cck16}
Lin Chen, Forrest~W Crawford, and Amin Karbasi.
\newblock Submodular variational inference for network reconstruction.
\newblock {\em arXiv preprint arXiv:1603.08616}, 2016.

\bibitem[CCLY19]{ccly19}
Michael~B Cohen, Ben Cousins, Yin~Tat Lee, and Xin Yang.
\newblock A near-optimal algorithm for approximating the john ellipsoid.
\newblock In {\em Conference on Learning Theory}, pages 849--873. PMLR, 2019.

\bibitem[CCPV07]{calinescu2007maximizing}
Gruia Calinescu, Chandra Chekuri, Martin P{\'a}l, and Jan Vondr{\'a}k.
\newblock Maximizing a submodular set function subject to a matroid constraint.
\newblock In {\em International Conference on Integer Programming and
  Combinatorial Optimization}, pages 182--196. Springer, 2007.

\bibitem[CCPV11]{calinescu2011maximizing}
Gruia Calinescu, Chandra Chekuri, Martin Pal, and Jan Vondr{\'a}k.
\newblock Maximizing a monotone submodular function subject to a matroid
  constraint.
\newblock {\em SIAM Journal on Computing}, 40(6):1740--1766, 2011.

\bibitem[CDKV16]{cdkv16}
L~Elisa Celis, Amit Deshpande, Tarun Kathuria, and Nisheeth~K Vishnoi.
\newblock How to be fair and diverse?
\newblock {\em arXiv preprint arXiv:1610.07183}, 2016.

\bibitem[CHHK18]{chen2018projection}
Lin Chen, Christopher Harshaw, Hamed Hassani, and Amin Karbasi.
\newblock Projection-free online optimization with stochastic gradient: From
  convexity to submodularity.
\newblock In {\em International Conference on Machine Learning}, pages
  814--823. PMLR, 2018.

\bibitem[CHK18]{chen2018online}
Lin Chen, Hamed Hassani, and Amin Karbasi.
\newblock Online continuous submodular maximization.
\newblock In {\em International Conference on Artificial Intelligence and
  Statistics}, pages 1896--1905. PMLR, 2018.

\bibitem[CLS19]{cls19}
Michael~B Cohen, Yin~Tat Lee, and Zhao Song.
\newblock Solving linear programs in the current matrix multiplication time.
\newblock In {\em STOC}, 2019.

\bibitem[CN22]{cn22}
Yeshwanth Cherapanamjeri and Jelani Nelson.
\newblock Uniform approximations for randomized hadamard transforms with
  applications.
\newblock {\em arXiv preprint arXiv:2203.01599}, 2022.

\bibitem[CSW{\etalchar{+}}23]{csw+23}
Yeshwanth Cherapanamjeri, Sandeep Silwal, David~P Woodruff, Fred Zhang, Qiuyi
  Zhang, and Samson Zhou.
\newblock Robust algorithms on adaptive inputs from bounded adversaries.
\newblock {\em arXiv preprint arXiv:2304.07413}, 2023.

\bibitem[DK11]{dk11}
Abhimanyu Das and David Kempe.
\newblock Submodular meets spectral: Greedy algorithms for subset selection,
  sparse approximation and dictionary selection.
\newblock {\em arXiv preprint arXiv:1102.3975}, 2011.

\bibitem[DK14]{djolonga2014map}
Josip Djolonga and Andreas Krause.
\newblock From map to marginals: Variational inference in bayesian submodular
  models.
\newblock {\em Advances in Neural Information Processing Systems}, 27, 2014.

\bibitem[DLS23a]{dls23_softmax}
Yichuan Deng, Zhihang Li, and Zhao Song.
\newblock Attention scheme inspired softmax regression.
\newblock {\em arXiv preprint arXiv:2304.10411}, 2023.

\bibitem[DLS23b]{dls23_sensing}
Yichuan Deng, Zhihang Li, and Zhao Song.
\newblock An improved sample complexity for rank-1 matrix sensing.
\newblock {\em arXiv preprint arXiv:2303.06895}, 2023.

\bibitem[DMNS06]{dmns06}
Cynthia Dwork, Frank McSherry, Kobbi Nissim, and Adam Smith.
\newblock Calibrating noise to sensitivity in private data analysis.
\newblock In {\em Theory of Cryptography: Third Theory of Cryptography
  Conference, TCC 2006, New York, NY, USA, March 4-7, 2006. Proceedings 3},
  pages 265--284. Springer, 2006.

\bibitem[DMS23]{dms23}
Yichuan Deng, Sridhar Mahadevan, and Zhao Song.
\newblock Randomized and deterministic attention sparsification algorithms for
  over-parameterized feature dimension.
\newblock {\em arxiv preprint: arxiv 2304.03426}, 2023.

\bibitem[DR{\etalchar{+}}14]{dr14}
Cynthia Dwork, Aaron Roth, et~al.
\newblock The algorithmic foundations of differential privacy.
\newblock {\em Foundations and Trends{\textregistered} in Theoretical Computer
  Science}, 9(3--4):211--407, 2014.

\bibitem[DRV10]{drv10}
Cynthia Dwork, Guy~N Rothblum, and Salil Vadhan.
\newblock Boosting and differential privacy.
\newblock In {\em 2010 IEEE 51st annual symposium on foundations of computer
  science}, pages 51--60. IEEE, 2010.

\bibitem[DSW22]{dsw22}
Yichuan Deng, Zhao Song, and Omri Weinstein.
\newblock Discrepancy minimization in input-sparsity time.
\newblock {\em arXiv preprint arXiv:2210.12468}, 2022.

\bibitem[Fei98]{feige1998threshold}
Uriel Feige.
\newblock A threshold of ln n for approximating set cover.
\newblock {\em Journal of the ACM (JACM)}, 45(4):634--652, 1998.

\bibitem[FFL{\etalchar{+}}25]{ffl+25}
Shiyuan Feng, Ying Feng, George~Zhaoqi Li, Zhao Song, David Woodruff, and
  Lichen Zhang.
\newblock On differential privacy for adaptively solving search problems via
  sketching.
\newblock In {\em ICML}, 2025.

\bibitem[GB10]{gb10}
Andrew Guillory and Jeff Bilmes.
\newblock Interactive submodular set cover.
\newblock {\em arXiv preprint arXiv:1002.3345}, 2010.

\bibitem[GK11a]{golovin2011adaptive}
Daniel Golovin and Andreas Krause.
\newblock Adaptive submodularity: Theory and applications in active learning
  and stochastic optimization.
\newblock {\em Journal of Artificial Intelligence Research}, 42:427--486, 2011.

\bibitem[GK11b]{gk11}
Daniel Golovin and Andreas Krause.
\newblock Adaptive submodularity: Theory and applications in active learning
  and stochastic optimization.
\newblock {\em Journal of Artificial Intelligence Research}, 42:427--486, 2011.

\bibitem[GKW{\etalchar{+}}13]{gabillon2013adaptive}
Victor Gabillon, Branislav Kveton, Zheng Wen, Brian Eriksson, and
  Shanmugavelayutham Muthukrishnan.
\newblock Adaptive submodular maximization in bandit setting.
\newblock {\em Advances in Neural Information Processing Systems}, 26, 2013.

\bibitem[GQSW22]{gqsw22}
Yeqi Gao, Lianke Qin, Zhao Song, and Yitan Wang.
\newblock A sublinear adversarial training algorithm.
\newblock {\em arXiv preprint arXiv:2208.05395}, 2022.

\bibitem[GRSD{\etalchar{+}}16]{gsdz+16}
Manuel Gomez-Rodriguez, Le~Song, Nan Du, Hongyuan Zha, and Bernhard
  Sch{\"o}lkopf.
\newblock Influence estimation and maximization in continuous-time diffusion
  networks.
\newblock {\em ACM Transactions on Information Systems (TOIS)}, 34(2):1--33,
  2016.

\bibitem[GS22]{gs22}
Yuzhou Gu and Zhao Song.
\newblock A faster small treewidth sdp solver.
\newblock {\em arXiv preprint arXiv:2211.06033}, 2022.

\bibitem[GSY23]{gsy23}
Yeqi Gao, Zhao Song, and Junze Yin.
\newblock An iterative algorithm for rescaled hyperbolic functions regression.
\newblock {\em arXiv preprint arXiv:2305.00660}, 2023.

\bibitem[GSYZ23a]{gsyz23_dp}
Yeqi Gao, Zhao Song, Xin Yang, and Yufa Zhou.
\newblock Differentially private attention computation.
\newblock {\em arXiv preprint arXiv:2305.04701}, 2023.

\bibitem[GSYZ23b]{gsyz23}
Yuzhou Gu, Zhao Song, Junze Yin, and Lichen Zhang.
\newblock Low rank matrix completion via robust alternating minimization in
  nearly linear time.
\newblock {\em arXiv preprint arXiv:2302.11068}, 2023.

\bibitem[HLL{\etalchar{+}}24]{hll+24}
Jerry Yao-Chieh Hu, Erzhi Liu, Han Liu, Zhao Song, and Lichen Zhang.
\newblock On differentially private string distances.
\newblock {\em arXiv preprint arXiv:2411.05750}, 2024.

\bibitem[HSWZ22]{hswz22}
Hang Hu, Zhao Song, Omri Weinstein, and Danyang Zhuo.
\newblock Training overparametrized neural networks in sublinear time.
\newblock {\em arXiv preprint arXiv:2208.04508}, 2022.

\bibitem[IM98]{im98}
Piotr Indyk and Rajeev Motwani.
\newblock Approximate nearest neighbors: towards removing the curse of
  dimensionality.
\newblock In {\em Proceedings of the thirtieth annual ACM symposium on Theory
  of computing}, pages 604--613, 1998.

\bibitem[JB11]{jb11}
Stefanie Jegelka and Jeff Bilmes.
\newblock Submodularity beyond submodular energies: coupling edges in graph
  cuts.
\newblock In {\em CVPR 2011}, pages 1897--1904. IEEE, 2011.

\bibitem[JLSW20]{jlsw20}
Haotian Jiang, Yin~Tat Lee, Zhao Song, and Sam Chiu-wai Wong.
\newblock An improved cutting plane method for convex optimization,
  convex-concave games and its applications.
\newblock In {\em STOC}, 2020.

\bibitem[JLSZ23]{jlsz23}
Haotian Jiang, Yin~Tat Lee, Zhao Song, and Lichen Zhang.
\newblock Convex minimization with integer minima in $\widetilde{O}(n^4)$ time.
\newblock {\em arXiv preprint arXiv:2304.03426}, 2023.

\bibitem[JSWZ21]{jswz21}
Shunhua Jiang, Zhao Song, Omri Weinstein, and Hengjie Zhang.
\newblock Faster dynamic matrix inverse for faster lps.
\newblock In {\em STOC}, 2021.

\bibitem[KG12]{krause2012near}
Andreas Krause and Carlos~E Guestrin.
\newblock Near-optimal nonmyopic value of information in graphical models.
\newblock {\em arXiv preprint arXiv:1207.1394}, 2012.

\bibitem[KKT03]{kkt03}
David Kempe, Jon Kleinberg, and {\'E}va Tardos.
\newblock Maximizing the spread of influence through a social network.
\newblock In {\em Proceedings of the ninth ACM SIGKDD international conference
  on Knowledge discovery and data mining}, pages 137--146, 2003.

\bibitem[KRGG09]{krause2009simultaneous}
Andreas Krause, Ram Rajagopal, Anupam Gupta, and Carlos Guestrin.
\newblock Simultaneous placement and scheduling of sensors.
\newblock In {\em 2009 International Conference on Information Processing in
  Sensor Networks}, pages 181--192. IEEE, 2009.

\bibitem[LB11a]{lin2011class}
Hui Lin and Jeff Bilmes.
\newblock A class of submodular functions for document summarization.
\newblock In {\em Proceedings of the 49th annual meeting of the association for
  computational linguistics: human language technologies}, pages 510--520,
  2011.

\bibitem[LB11b]{lb11}
Hui Lin and Jeff Bilmes.
\newblock Word alignment via submodular maximization over matroids.
\newblock In {\em Proceedings of the 49th Annual Meeting of the Association for
  Computational Linguistics: Human Language Technologies}, pages 170--175,
  2011.

\bibitem[LDS{\etalchar{+}}21]{lds+21}
Bo~Liu, Ming Ding, Sina Shaham, Wenny Rahayu, Farhad Farokhi, and Zihuai Lin.
\newblock When machine learning meets privacy: A survey and outlook.
\newblock {\em ACM Computing Surveys (CSUR)}, 54(2):1--36, 2021.

\bibitem[LHR{\etalchar{+}}24]{lhr+24}
Erzhi Liu, Jerry Yao-Chieh Hu, Alex Reneau, Zhao Song, and Han Liu.
\newblock Differentially private kernel density estimation.
\newblock {\em arXiv preprint arXiv:2409.01688}, 2024.

\bibitem[LKG{\etalchar{+}}07]{leskovec2007cost}
Jure Leskovec, Andreas Krause, Carlos Guestrin, Christos Faloutsos, Jeanne
  VanBriesen, and Natalie Glance.
\newblock Cost-effective outbreak detection in networks.
\newblock In {\em Proceedings of the 13th ACM SIGKDD international conference
  on Knowledge discovery and data mining}, pages 420--429, 2007.

\bibitem[LL23]{ll23}
Xiaoyun Li and Ping Li.
\newblock Differentially private one permutation hashing and bin-wise
  consistent weighted sampling.
\newblock {\em arXiv preprint arXiv:2306.07674}, 2023.

\bibitem[LSZ19]{lsz19}
Yin~Tat Lee, Zhao Song, and Qiuyi Zhang.
\newblock Solving empirical risk minimization in the current matrix
  multiplication time.
\newblock In {\em Conference on Learning Theory (COLT)}, pages 2140--2157.
  PMLR, 2019.

\bibitem[LSZ23]{lsz23}
Zhihang Li, Zhao Song, and Tianyi Zhou.
\newblock Solving regularized exp, cosh and sinh regression problems.
\newblock {\em arXiv preprint, 2303.15725}, 2023.

\bibitem[MBK{\etalchar{+}}15]{mbk+15}
Baharan Mirzasoleiman, Ashwinkumar Badanidiyuru, Amin Karbasi, Jan Vondr{\'a}k,
  and Andreas Krause.
\newblock Lazier than lazy greedy.
\newblock In {\em Proceedings of the AAAI Conference on Artificial
  Intelligence}, volume~29, 2015.

\bibitem[MBMW05]{munagala2005pipelined}
Kamesh Munagala, Shivnath Babu, Rajeev Motwani, and Jennifer Widom.
\newblock The pipelined set cover problem.
\newblock In {\em International Conference on Database Theory}, pages 83--98.
  Springer, 2005.

\bibitem[MH18]{mehr2018submodular}
Negar Mehr and Roberto Horowitz.
\newblock A submodular approach for optimal sensor placement in traffic
  networks.
\newblock In {\em 2018 Annual American Control Conference (ACC)}, pages
  6353--6358. IEEE, 2018.

\bibitem[MT07]{mt07}
Frank McSherry and Kunal Talwar.
\newblock Mechanism design via differential privacy.
\newblock In {\em 48th Annual IEEE Symposium on Foundations of Computer Science
  (FOCS'07)}, pages 94--103. IEEE, 2007.

\bibitem[NJB05]{narasimhan2005q}
Mukund Narasimhan, Nebojsa Jojic, and Jeff~A Bilmes.
\newblock Q-clustering.
\newblock {\em Advances in Neural Information Processing Systems}, 18, 2005.

\bibitem[NWF78]{nwf78}
George~L Nemhauser, Laurence~A Wolsey, and Marshall~L Fisher.
\newblock An analysis of approximations for maximizing submodular set
  functions—i.
\newblock {\em Mathematical programming}, 14(1):265--294, 1978.

\bibitem[PHK{\etalchar{+}}23]{phk+23}
Natalia Ponomareva, Hussein Hazimeh, Alex Kurakin, Zheng Xu, Carson Denison,
  H~Brendan McMahan, Sergei Vassilvitskii, Steve Chien, and Abhradeep~Guha
  Thakurta.
\newblock How to dp-fy ml: A practical guide to machine learning with
  differential privacy.
\newblock {\em Journal of Artificial Intelligence Research}, 77:1113--1201,
  2023.

\bibitem[QJS{\etalchar{+}}22]{qjs+22}
Lianke Qin, Rajesh Jayaram, Elaine Shi, Zhao Song, Danyang Zhuo, and Shumo Chu.
\newblock Adore: Differentially oblivious relational database operators.
\newblock {\em arXiv preprint arXiv:2212.05176}, 2022.

\bibitem[QRS{\etalchar{+}}22]{qrs+22}
Lianke Qin, Aravind Reddy, Zhao Song, Zhaozhuo Xu, and Danyang Zhuo.
\newblock Adaptive and dynamic multi-resolution hashing for pairwise
  summations.
\newblock In {\em BigData}, 2022.

\bibitem[QSZ23]{qsz23}
Lianke Qin, Zhao Song, and Ruizhe Zhang.
\newblock A general algorithm for solving rank-one matrix sensing.
\newblock {\em arXiv preprint arXiv:2303.12298}, 2023.

\bibitem[QSZZ23]{qszz23}
Lianke Qin, Zhao Song, Lichen Zhang, and Danyang Zhuo.
\newblock An online and unified algorithm for projection matrix vector
  multiplication with application to empirical risk minimization.
\newblock In {\em AISTATS}, 2023.

\bibitem[QWH24]{qwh24}
Tao Qi, Huili Wang, and Yongfeng Huang.
\newblock Towards the robustness of differentially private federated learning.
\newblock In {\em Proceedings of the AAAI Conference on Artificial
  Intelligence}, volume 38:18, pages 19911--19919, 2024.

\bibitem[SAMB24]{samb24}
Tanmay Singh, Harshvardhan Aditya, Vijay~K Madisetti, and Arshdeep Bahga.
\newblock Whispered tuning: Data privacy preservation in fine-tuning llms
  through differential privacy.
\newblock {\em Journal of Software Engineering and Applications}, 17(1):1--22,
  2024.

\bibitem[SBB{\etalchar{+}}14]{sbb+14}
Adish Singla, Ilija Bogunovic, G{\'a}bor Bart{\'o}k, Amin Karbasi, and Andreas
  Krause.
\newblock Near-optimally teaching the crowd to classify.
\newblock In {\em International Conference on Machine Learning}, pages
  154--162. PMLR, 2014.

\bibitem[SDP{\etalchar{+}}19]{shen2019submodular}
Jianbing Shen, Xingping Dong, Jianteng Peng, Xiaogang Jin, Ling Shao, and Fatih
  Porikli.
\newblock Submodular function optimization for motion clustering and image
  segmentation.
\newblock {\em IEEE transactions on neural networks and learning systems},
  30(9):2637--2649, 2019.

\bibitem[SG08]{streeter2008online}
Matthew Streeter and Daniel Golovin.
\newblock An online algorithm for maximizing submodular functions.
\newblock {\em Advances in Neural Information Processing Systems}, 21, 2008.

\bibitem[SGS08]{streeter2008combining}
Matthew Streeter, Daniel Golovin, and Stephen~F Smith.
\newblock Combining multiple constraint solvers: Results on the cpai’06
  competition data.
\newblock {\em of: Proceedings of the Second International CSP Solver
  Competition}, pages 11--18, 2008.

\bibitem[SJS10]{sjs10}
Bastian Steudel, Dominik Janzing, and Bernhard Sch{\"o}lkopf.
\newblock Causal markov condition for submodular information measures.
\newblock {\em arXiv preprint arXiv:1002.4020}, 2010.

\bibitem[Son19]{song19}
Zhao Song.
\newblock {\em Matrix theory: optimization, concentration, and algorithms}.
\newblock The University of Texas at Austin, 2019.

\bibitem[SSX21]{ssx21}
Anshumali Shrivastava, Zhao Song, and Zhaozhuo Xu.
\newblock Sublinear least-squares value iteration via locality sensitive
  hashing.
\newblock {\em arXiv preprint arXiv:2105.08285}, 2021.

\bibitem[Svi04]{sviridenko2004note}
Maxim Sviridenko.
\newblock A note on maximizing a submodular set function subject to a knapsack
  constraint.
\newblock {\em Operations Research Letters}, 32(1):41--43, 2004.

\bibitem[SWYZ23]{swyz23}
Zhao Song, Yitan Wang, Zheng Yu, and Lichen Zhang.
\newblock Sketching for first order method: Efficient algorithm for
  low-bandwidth channel and vulnerability.
\newblock In {\em ICML}, 2023.

\bibitem[SXZ22]{sxz22}
Zhao Song, Zhaozhuo Xu, and Lichen Zhang.
\newblock Speeding up sparsification using inner product search data
  structures.
\newblock {\em arXiv preprint arXiv:2204.03209}, 2022.

\bibitem[SY21]{sy21}
Zhao Song and Zheng Yu.
\newblock Oblivious sketching-based central path method for linear programming.
\newblock In {\em International Conference on Machine Learning}, pages
  9835--9847. PMLR, 2021.

\bibitem[SYYZ22]{syyz22}
Zhao Song, Xin Yang, Yuanyuan Yang, and Tianyi Zhou.
\newblock Faster algorithm for structured john ellipsoid computation.
\newblock {\em arXiv preprint arXiv:2211.14407}, 2022.

\bibitem[SYYZ23]{syyz23_dp}
Zhao Song, Xin Yang, Yuanyuan Yang, and Lichen Zhang.
\newblock Sketching meets differential privacy: fast algorithm for dynamic
  kronecker projection maintenance.
\newblock In {\em ICML}, pages 32418--32462. PMLR, 2023.

\bibitem[SZZ21]{szz21}
Zhao Song, Lichen Zhang, and Ruizhe Zhang.
\newblock Training multi-layer over-parametrized neural network in subquadratic
  time.
\newblock {\em arXiv preprint arXiv:2112.07628}, 2021.

\bibitem[TIWB14]{tiwb14}
Sebastian Tschiatschek, Rishabh~K Iyer, Haochen Wei, and Jeff~A Bilmes.
\newblock Learning mixtures of submodular functions for image collection
  summarization.
\newblock {\em Advances in neural information processing systems}, 27, 2014.

\bibitem[WIB15]{wib15}
Kai Wei, Rishabh Iyer, and Jeff Bilmes.
\newblock Submodularity in data subset selection and active learning.
\newblock In {\em International conference on machine learning}, pages
  1954--1963. PMLR, 2015.

\bibitem[Wol82]{wolsey1982analysis}
Laurence~A Wolsey.
\newblock An analysis of the greedy algorithm for the submodular set covering
  problem.
\newblock {\em Combinatorica}, 2(4):385--393, 1982.

\bibitem[YFT16]{yu2016linear}
Baosheng Yu, Meng Fang, and Dacheng Tao.
\newblock Linear submodular bandits with a knapsack constraint.
\newblock In {\em Thirtieth AAAI Conference on Artificial Intelligence}, 2016.

\bibitem[YG11]{yue2011linear}
Yisong Yue and Carlos Guestrin.
\newblock Linear submodular bandits and their application to diversified
  retrieval.
\newblock {\em Advances in Neural Information Processing Systems}, 24, 2011.

\bibitem[ZBC{\etalchar{+}}16]{zbcb+16}
Yuanxing Zhang, Yichong Bai, Lin Chen, Kaigui Bian, and Xiaoming Li.
\newblock Influence maximization in messenger-based social networks.
\newblock In {\em 2016 IEEE Global Communications Conference (GLOBECOM)}, pages
  1--6. IEEE, 2016.

\bibitem[Zha22]{z22}
Lichen Zhang.
\newblock Speeding up optimizations via data structures: Faster search, sample
  and maintenance.
\newblock Master's thesis, Carnegie Mellon University, 2022.

\bibitem[ZS16]{zs16}
Yuxun Zhou and Costas~J Spanos.
\newblock Causal meets submodular: Subset selection with directed information.
\newblock {\em Advances in Neural Information Processing Systems}, 29, 2016.

\end{thebibliography}
\fi


\newpage
\onecolumn
\appendix

\begin{center}
    \textbf{\LARGE Appendix }
\end{center}



\paragraph{Roadmap.}
The appendix of our paper is organized as follows. 
Section~\ref{sec:app_preli} gives the preliminary tools and some tools from previous works, . 
Section~\ref{sec:app_benchmark} complements the omitted details of our greedy benchmark. 
Section~\ref{sec:app_data_structure} complements the omitted details of our data structures. Section~\ref{sec:app_alg_lsh} shows the implementation with locality sensitive hashing. 
Section~\ref{sec:extension} demonstrate how to utilize our data structure to solve extension forms of fast submodular function maximization problem. 
Section~\ref{sec:app_extension} shows how our work can be extended to typical constraints and online setting. Section~\ref{sec:combination} utilizes our data structures to obtain a formal submodular maximization algorithm. 

\section{Preliminary}\label{sec:app_preli}

\subsection{Notation}  
For any positive $n$, we use $[n]$ to denote the set $\{1,2,\cdots,n\}$. 
For a square matrix $A$, we use $\tr[A]$ to denote the trace of $A$.
For a set $S\subseteq[n]$ , we use $\overline{S}$ to denote the complementary set of $S$, which is equivalent to $[n]\backslash S$.
If there is no ambiguity, we will omit $[n]$.
We use $n$ to denote the number of elements in the ground set $[n]$.
We use $d$ to denote the dimension of each vector.
For $t\ge 1$, we use $S_t$ to denote the set of selected elements at time $1,2,\cdots,t$.

For two $d$-dimensional vectors $u,v$, we use $\langle u,v \rangle$ to denote their inner product. For two $n\times m$ matrices $A,B$, we use $\langle A,B \rangle$ to denote their inner product, that is, their elementary-wise product. For any function $f$, we will use $\wt{O}(f)$ to denote $O(f \cdot \poly (\log f))$.

For notational simplicity, we use $\ov{S}_t$ to denote $[n] \backslash S_t$. 
In submodular maximization, in each iteration, we are greedily updating the an object $S_{t+1} \subset [n] $ in the following sense

\begin{align*}
    S_{t+1} \gets S_t \cup \{ \arg\max_{ j \in \ov{S}_t } f( S_t \cup \{ j \} ) \}
\end{align*}
where $f$ is a submodular function.

Let us decompose $f$ as sum of increments, e.g.
\begin{align*}
    f(S) = f(S-\{j\}) + \Delta_f(j|S-\{j\})
\end{align*}

Let $Q = S - \{j\}$.
Let us represent $\Delta_f(j|S-\{j\})$ as
\begin{align*}
    \Delta_f(j|S-\{j\}) = u_j^\top \cdot h( Q ) \cdot u_j
\end{align*}
where $h : 2^{[n]} \rightarrow \R^{d \times d}$ and the output of function $h$ should be a PSD matrix when $f$ is monotone.

\subsection{Previous Tools}

\paragraph{Fast multiplication.}\label{sec:previous_tools:fast_mutiplication}
In our algorithm, we need to operate dense matrix multiplication.
Currently, the fastest matrix multiplication achieves $O(n^{2.373})$ time complexity (\cite{aw21}).
\begin{theorem}[Matrix Multiplication(\cite{aw21})]\label{thm:fast_matrix_multiplication}
Let $\Tmat(n)$ denote the number of operations used to compute matrix product $AB$ where $A,B \in \mathbb{R}^{n\times n}$.
There exists an algorithm computing $AB$ with $\Tmat(n)=O(n^\omega)$ for some $\omega < 2.373$.
\end{theorem}

\section{Omitted Details of Benchmark}\label{sec:app_benchmark}
In this section, we provide omitted details of benchmark. Section~\ref{sec:app_benchmark:versions_greedy} present details of different version of greedy algorithm.



\subsection{Different Version of Greedy Algorithm}\label{sec:app_benchmark:versions_greedy}

The greedy algorithm benchmark is as described in Algorithm~\ref{alg:greedy}.

\begin{algorithm}[!ht]\caption{Greedy algorithm benchmark}\label{alg:greedy}
\begin{algorithmic}[1]
\Procedure{GreedyAlgorithm}{submodular function $f$}
    \State $S_0 \gets \emptyset$
    \For{$t=0 \to k-1$}
        \State $j \gets \arg\max_i\{ f(S_t \cup \{i\}) \}$
        \State $S_{t+1} \gets S_t \cup \{j\}$
    \EndFor
    \State \Return $S_k$
\EndProcedure
\end{algorithmic}
\end{algorithm}

\begin{algorithm}[!ht]\caption{Naive implementation of the greedy selection algorithm}\label{alg:trivial}
\begin{algorithmic}[1]
\Procedure{Naive-Alg}{$n,k,\{u_1, \cdots, u_n \} \subseteq \R^d$}
    \State $S_0 \gets \emptyset$
    \For{$t=0 \to k-1$}
        \State $A_t \gets h (S_t)$, $w_{\max} \gets -1$, $j_{\max} \gets -1$
        \For {$j\in  \ov{S}_t$} \label{line:enumeration}
            \State $w_j \gets u_j^{\top} A_t u_j$
            \If {$w_j > w_{\max}$}
                \State $w_{\max} \gets w_j$,  $j_{\max} \gets j$
            \EndIf 
        \EndFor 
        \State $S_{t+1} \gets S_t \cup \{ j_{\max} \}$
    \EndFor
    \State \Return $S_k$
\EndProcedure
\end{algorithmic}
\end{algorithm}

\begin{algorithm}[!ht]\caption{Greedy Algorithm with Batch}\label{alg:trivial_v2}
\begin{algorithmic}[1]
\Procedure{BatchAlg}{$n,k,\{u_1, \cdots, u_n \} \subset \R^d$}
    \State $S_0 \gets \emptyset$
    \For{$t=0 \to k-1$}
        \State $A_t \gets h (S_t)$
        \State $w_{\max} \gets -1$
        \State $j_{\max} \gets -1$
        \For {$b=1$ to $m$} \label{line:clever_enumeration}
            \State $U \gets [~u_{(b-1)d + 1}~u_{(b-1) d+2}~\cdots~u_{(b-1)d+d}~]$
            \State $M \gets U^{\top} A_t U$
            \For {$j=1$ to $d$}
                \State $j_{\mathrm{tmp}} \gets (b-1) \cdot d + j$
                \If {$M_{j,j} > w_{\max}$ and $ j_{\mathrm{tmp}} \in \ov{S}_t$}
                    \State $w_{\max} \gets M_{j,j}$
                    \State $j_{\max} \gets j_{\mathrm{tmp}}$ 
                \EndIf 
            \EndFor 
        \EndFor 
        \State $S_{t+1} \gets S_t \cup \{ j_{\max} \}$
    \EndFor
    \State \Return $S_k$
\EndProcedure
\end{algorithmic}
\end{algorithm}

\section{Details and Omitted Proofs of Our Data Structures}\label{sec:app_data_structure}

This section is organized as follows: Section~\ref{sec:app_data_structure:details} present the detials of our data structures, Ipe and ds. 
Section \ref{sec:app_proof_IPE} complements the omitted proofs of our inner product estimation data structure (see Section \ref{sec:ds}), which can return the approximate inner products between a query vector $v$ and a list of given vectors $u_1,u_2,\cdots,u_n$; Section \ref{sec:app_faster_ds} complements the omitted proofs of our faster quadratic form data structure (see Section \ref{sec:ds}), which can approximately return the quadratic form $u_j^\top A u_j$ for a query matrix $M$ and given vectors $u_1,u_2,\cdots,u_n$.  

\subsection{Details of Data Structures and Algorithms}\label{sec:app_data_structure:details}

\begin{algorithm}[!ht]\caption{Dynamic inner product estimation}\label{alg:inner_product_estimation}
\begin{algorithmic}[1]
\State {\bf data structure} \textsc{DynamicInnerProductEstimation}
\Comment{Theorem \ref{thm:IPE}}
\State {\bf members}
\State \hspace{4mm} ADE Ade
\State {\bf end members}
\State 
\Procedure{Init}{$x_1,x_2,\cdots,x_n,\epsilon,\delta$} \Comment{Lemma \ref{lemma:app_IPE_init} } 
    \State $\epsilon'=\frac{2\epsilon}{3D}$, Ade.\textsc{Init}$(Q(x_1),Q(x_2),\cdots,Q(x_n),\epsilon',\delta)$
\EndProcedure 
\State 
\Procedure{Update}{$i,z$} \Comment{Lemma \ref{lemma:app_IPE_update} } 
    \State Ade.\textsc{Update}$(i,Q(z))$
\EndProcedure 
\State 
\Procedure{Query}{$q$} \Comment{Lemma \ref{lemma:app_IPE_query} } 
    \State $\tilde{d}_1, \tilde{d}_2, \cdots, \tilde{d}_n$ = Ade.\textsc{Query}$(P(q))$
    \State $\tilde{w}_i = 1-\frac{1}{2}\tilde{d}_i^2$ for $i \in [n]$
    \State \Return $\{\tilde{w}_i\}_{i=1}^n$
\EndProcedure 
\State {\bf data structure}
\end{algorithmic}
\end{algorithm}

\begin{algorithm}[!ht]\caption{Data Structure ds} \label{alg:ds_v2}
\begin{algorithmic}[1]
    \State {\bf data structure} \textsc{ds}
    \State {\bf members}
    \State \hspace{4mm} $S$ (Candidate index set)
    \State \hspace{4mm} \textsc{DynamicInnerProductEstimation} $\text{Ipe}$ \Comment{Algorithm~\ref{alg:inner_product_estimation}}
    \State {\bf end members}
    \State
    \Procedure{\textsc{Init}}{$\{u_1,u_2,\cdots,u_n\}\subset \mathbb{R}^d, \epsilon > 0, \delta > 0$}
        \State $S \gets [n]$,
        $v_i \gets [ ~ (u_i)_1 \cdot u_i^\top ~~ (u_i)_2 \cdot u_i^\top ~~ \cdots ~~ (u_i)_d \cdot u_i^\top ~]^\top$ for $i \in [n]$
        \State $\text{Ipe}$.\textsc{Init}$(\{v_1,v_2,\cdots v_n\}, \epsilon, \delta)$ \label{line:Ipe_init_v2}
    \EndProcedure
    \State
    \Procedure{\textsc{Query}}{$M\in \mathbb{R}^{d\times d}$}
        \State $M \gets [ ~ m_1 ~ m_2 ~ \cdots ~ m_d ~ ]$, $m \gets [ ~ m_1^\top ~ m_2^\top ~ \cdots ~ m_d^\top ~ ]^\top$
        \State $\{\tilde{w}_{1}, \tilde{w}_{2},\cdots,\tilde{w}_{n}\} \gets \text{Ipe}$.\textsc{Query}$(m)$ \label{line:Ipe_query_v2}
        \State \Return $\arg\max_{j\in S}\{\tilde{w}_j\}$ \label{line:return}
    \EndProcedure
    \State
    \Procedure{\textsc{Delete}}{$i\in[n]$}
        \State $S \gets S \backslash \{i\}$
    \EndProcedure
    \State {\bf end data structure}
\end{algorithmic}
\end{algorithm} 

\begin{algorithm}[!ht]\caption{Fast Submodular Function Maximization Algorithm with Data Structure.}\label{alg:ours}
\begin{algorithmic}[1]
\Procedure{FastSubmodularMaximization}{$n,k,h,\{u_1, \cdots, u_n \} \subset \R^d$}
    \State \text{ds}.\textsc{Init}($u_1, \cdots u_n$)
    \State $S_0 \gets \emptyset$
    \For{$t=0 \to k-1$}
        \State $A_t \gets h (S_t)$
        \State $j \gets \text{ds}.\textsc{Query}(A_t)$
        \State $S_{t+1} \gets S_t \cup \{ j \}$
        \State $\text{ds}.\textsc{Del}(j)$
     \EndFor
    \State \Return $S_k$
\EndProcedure
\end{algorithmic}
\end{algorithm}

\subsection{Proof of Theorem~\ref{thm:IPE}: Inner Product Estimation} \label{sec:app_proof_IPE}

The goal of this section is to prove the following theorem.

\begin{theorem}[Restatement of Theorem \ref{thm:IPE}]\label{thm:app_IPE}
There is a data structure uses $\tilde{O}(\epsilon^{-2}D^2nd\log(1/\delta))$ space for the Online Approximate Adaptive Inner Product Estimation Problem with the following procedures:
\begin{itemize}
    \item \textsc{Init}$(  \{x_1, x_2, \dots, x_n\}\subset \R^d, \epsilon \in (0,1), \delta \in (0,1))$: Given data points $\{x_1, x_2, \dots, x_n\}\subset \R^d$ ($\| x_i \| \leq D$, for all $i \in [n]$), an accuracy parameter $\epsilon$ and a failure probability $\delta$ as input, the data structure preprocesses in time $\wt{O}(\epsilon^{-2} D^2 n d \log(1/\delta))$.
    \item \textsc{Update}$(i \in [n], z \in \R^d)$: Given index $i$ and coordinate $z$, the data structure replace $x_i$ by $z$ in time $\tilde{O}(\epsilon^{-2}D^2d\log(1/\delta))$.
    \item \textsc{Query}$(q \in \R^d)$: Given a query point $q \in \R^d$ (where $\| q \| \leq 1$), the \textsc{Query} operation takes $q$ as input and approximately estimates the inner product of $q$ and all the data points $\{x_1, x_2, \dots, x_n\}\subset \R^d$ in time $\tilde{O}(\epsilon^{-2}D^2(n+d)\log(1/\delta))$ i.e. it provides a set of estimates $\{\tilde{w}_i\}_{i=1}^n$ such that:
    \begin{align*}
       \forall i \in[n], \langle q,x_{i}\rangle - \epsilon \leq \tilde{w}_{i} \leq \langle q,x_{i}\rangle + \epsilon
    \end{align*}
     with probability at least $1 -\delta$, even for a sequence of adaptively chosen queries.
\end{itemize}
\end{theorem}
\begin{proof}
We prove Theorem \ref{thm:app_IPE} by proving the following three lemmas Lemma \ref{lemma:app_IPE_init}, Lemma \ref{lemma:app_IPE_update} and Lemma \ref{lemma:app_IPE_query}, where Lemma \ref{lemma:app_IPE_init} focuses on procedure \textsc{Init}, Lemma \ref{lemma:app_IPE_update} focuses on procedure \textsc{Update} and Lemma \ref{lemma:app_IPE_query} focuses on procedure \textsc{Query}.
\end{proof}


\begin{lemma}[Running time of initialization.]\label{lemma:app_IPE_init}
The procedure \textsc{Init} in Algorithm \ref{alg:inner_product_estimation} runs in time $\tilde{O}(\epsilon^{-2}D^2nd\log(1/\delta))$.
\end{lemma}

\begin{proof}
This reason is that the member Ade has \textsc{Init} time
\begin{align*}
\tilde{O}(\epsilon'^{-2}nd\log(1/\delta))=\tilde{O}(\epsilon^{-2}D^2nd\log(1/\delta))
\end{align*}
according to Theorem \ref{thm:cn22}.
\end{proof}

\begin{lemma}[Running time of Update. ]\label{lemma:app_IPE_update}
The procedure \textsc{Update} in Algorithm \ref{alg:inner_product_estimation}  
runs in time $\tilde{O}(\epsilon^{-2}D^2d\log(1/\delta))$.
\end{lemma}

\begin{proof}
The reason is that the member Ade has \textsc{Update} time $\tilde{O}(\epsilon'^{-2}d\log(1/\delta))=\tilde{O}(\epsilon^{-2}D^2d\log(1/\delta))$ by Theorem \ref{thm:cn22}  
\end{proof}


\begin{lemma}[Correctness and running time of Query.]\label{lemma:app_IPE_query}
The procedure \textsc{Query} in Algorithm \ref{alg:inner_product_estimation} outputs $\{\tilde{w}_i\}_{i=1}^n$ correctly and runs in time 
\begin{align*}
    \tilde{O}(\epsilon^{-2}D^2(n+d)\log(1/\delta)) .
\end{align*}
\end{lemma}

\begin{proof}
For the correctness part of procedure \textsc{Query}, let us denote $\|Q(x_i)-P(y)\|_2$ by $d_i$, using asymmetric transformation we have
\begin{align*}
    \langle x_i,y \rangle &= D - \frac{1}{2} D d_i^2,
\end{align*}
and by Theorem \ref{thm:cn22}, Ade.\textsc{Query}($P(y)$) gives a series of $\{\tilde{d_i}\}_{i=1}^n$ with 
\begin{align*}
(1-\epsilon')d_i \le \tilde{d_i} \le (1+\epsilon')d_i,
\end{align*}
thus 
\begin{align*}
    \tilde{w_i} = D-\frac{D}{2}\tilde{d_i}^2 \le D-\frac{D}{2}(1-\epsilon')^2d_i^2 \le \langle x_i,y \rangle + \frac{D}{2}(2\epsilon'-\epsilon'^2)d_i^2\le \langle x_i,y \rangle + \epsilon,
\end{align*}
\begin{align*}
    \tilde{w_i} = D-\frac{D}{2}\tilde{d_i}^2 \ge D-\frac{D}{2}(1+\epsilon')^2d_i^2 \ge \langle x_i,y \rangle - \frac{D}{2}(2\epsilon'+\epsilon'^2)d_i^2\ge \langle x_i,y \rangle - \epsilon,
\end{align*}
where we use $\frac{D}{2}(2\epsilon'-\epsilon'^2)\le D\epsilon' \le \epsilon$ and $\frac{D}{2}(2\epsilon'+\epsilon'^2)\le \frac{3}{2}D\epsilon' \le \epsilon$.
And for the running time, it holds because the member Ade has \textsc{Query} time
\begin{align*}
\tilde{O}(\epsilon'^{-2}(n+d)\log(1/\delta))=\tilde{O}(\epsilon^{-2}D^2(n+d)\log(1/\delta))
\end{align*}
according to Theorem \ref{thm:cn22}.
\end{proof}

\subsection{Proof of Theorem~\ref{thm:our_ds2}: Fast Quadratic Form Search Data Structure} \label{sec:app_faster_ds}

\begin{theorem}[Restatement of Theorem~\ref{thm:our_ds2}] \label{thm:app_our_ds2}
There exists a data structure uses $\tilde{O}(\epsilon^{-2}D^2nd^2\log(1/\delta))$ space with the following procedures:
\begin{itemize}
    \item \textsc{Init}{$(\{u_1,u_2,\cdots,u_n\}\subset \mathbb{R}^d, \epsilon\in(0,1), \delta\in(0,1))$}. Given vectors $u_1,u_2,\cdots,u_n$ satisfying $\|u_i\|_2 \le D$ for any $i\in [n]$, accuracy parameter $\epsilon$ and failure probability $\delta$, the data structure initializes with time $\tilde{O}(\epsilon^{-2}D^2nd^2\log(1/\delta))$.
    \item \textsc{Query}{$(M\in \mathbb{R}^{d\times d})$}. Given a $d\times d$ matrix $M$, the data structure outputs a $j_0\in\ov{S}_t$ 
    such that $u_{j_0}^\top M u_{j_0} \ge \max_{j}\{ u_j^\top M u_j \} - 2\epsilon$ with probability at least $1-\delta$ in time $\tilde{O}(\epsilon^{-2}D^2(n+d^2)\log(1/\delta))$. 
    \item \textsc{Delete}{$(i\in[n])$}. Given index $i$, the data structure deletes $u_i$ from the candidate index set in time $O(\log n)$.
\end{itemize}
\end{theorem}

\begin{proof}
This theorem follows from Lemma \ref{lem:app_ds_correctness_v2}, Lemma \ref{lem:app_ds_init_v2}, Lemma \ref{lem:app_ds_query_v2} and Lemma \ref{lem:app_ds_delete_v2}.
\end{proof}


\begin{lemma}[Correctness.] \label{lem:app_ds_correctness_v2}
Procedure \textsc{QUERY} outputs a $j_0\in S$ such that $u_{j_0}^\top M u_{j_0} \ge \max_{j\in \overline{S}_t}\{ u_j^\top M u_j \} - 2\epsilon$ with probability at least $1-\delta$.
\end{lemma}
\begin{proof}
By Theorem \ref{thm:IPE}, line \ref{line:Ipe_query_v2} returns a list $\{\tilde{w}_{1},\tilde{w}_{2},\cdots,\tilde{w}_{n}\}$ such that $\langle m, v_i \rangle - \epsilon \le \tilde{w}_{i} \le \langle m, v_i \rangle + \epsilon$ with probability at least $1 - \delta$. Note that 
\begin{align*}
    \langle m, v_i \rangle = \sum_{j=1}^d \langle m_j, (u_i)_j \cdot u_i\rangle = \langle M, u_iu_i^\top \rangle = \tr[Mu_iu_i^\top] = u_i^\top M u_i,
\end{align*}
where the first step follows from the definitions of $m$ and $v_i$, the second step follows from $(u_i)_j \cdot u_i$ equals the $j$-th column of $u_iu_i^\top$, the third step follows from $\langle A,B \rangle = \tr[AB^\top]$, 
and the last step follows from $\tr[AB] = \tr[BA]$. Thus 
\begin{align}
    u_i^\top M u_i - \epsilon \le \tilde{w}_i \le u_i^\top M u_i + \epsilon \label{eq:tilde_w_app}
\end{align}
with probability at least $1-\delta$. Hence procedure \textsc{Update} outputs a $j_0$ such that 
\begin{align*}
    u_{j_0}^\top M u_{j_0} \le \tilde{w}_{j_0} + \epsilon = \max_{j\in S}\{\tilde{w}_{j}\} + \epsilon \le u_j^\top M u_j + 2\epsilon
\end{align*}
with probability $1-\delta$, where the first step follows from Equation \eqref{eq:tilde_w_app}, the second step follows from line \ref{line:return} in Algorithm \ref{alg:ds_v2}, and the last step follows from Equation \eqref{eq:tilde_w_app}.
\end{proof}


\begin{lemma}[Init time.] \label{lem:app_ds_init_v2}
Procedure \textsc{Init} takes $\tilde{O}(\epsilon^{-2}D^2nd^2\log(1/\delta))$ time.
\end{lemma}

\begin{proof}
By Theorem \ref{thm:IPE}, $\text{Ipe}$.\textsc{Init} in line \ref{line:Ipe_init_v2} takes  $\tilde{O}(\epsilon^{-2}D^2nd^2\log(1/\delta))$ time, since $\dim(v_i) = d^2$ for each $i\in [n]$.
\end{proof}


\begin{lemma}[Query time.] \label{lem:app_ds_query_v2}
Procedure \textsc{Query} takes $\tilde{O}(\epsilon^{-2}D^2(n+d^2)\log(1/\delta))$ time.
\end{lemma}

\begin{proof}
By Theorem \ref{thm:IPE}, line \ref{line:Ipe_query_v2} takes $\tilde{O}(\epsilon^{-2}D^2(n+d^2)\log(1/\delta))$ time.
\end{proof}

\begin{lemma}[Delete time.] \label{lem:app_ds_delete_v2}
Procedure \textsc{Delete} takes $O(\log n)$ time to delete an index from candidate index set.
\end{lemma}
\begin{proof}
This holds trivially since it takes $O(\log n)$ time to find an index in the candidate index set $S$.
\end{proof}

\section{An extension: Fast Submodular Maximization Algorithm with LSH} \label{sec:app_alg_lsh}
In this section, we present an alternative implementation of fast submodular maximization by using locality sensitive hashing technique.
In Section~\ref{sec:app_alg_lsh:preli}, we review previous results about locality sensitive hashing and approximate near neighbor in \cite{im98,ssx21,ar15}.  In Section~\ref{sec:app_alg_lsh:result}, we present our results of greedy selection via locality sensitive hashing. In Section~\ref{sec:app_alg_lsh:details}, we present details of LSH data structure. In Section~\ref{sec:app_alg_lsh:LSH_problem}, we present the missing proof of Corollary~\ref{cor:ds_LSH_submodular_problem}.

\subsection{Preliminary}\label{sec:app_alg_lsh:preli}

Recent progress in locality sensitive hashing would also help us design faster algorithms.
Informally, we call a family of hash functions sensitive when it tends to project points closed to each other to the same value and tends to project points far away from each other to different value.
The formal definition of sensitive hashing is given in definition \ref{def:locality_sensitive_hashing}.
 
\begin{definition}[Locality sensitive hashing, \cite{im98}]\label{def:locality_sensitive_hashing}
Given parameter $R>0,c>1,0<p_2<p_1<1$, we say a family of hash functions ${\cal H}$ is $(R,cR,p_1,p_2)$-sensitive if it satisfies the following conditions:
\begin{itemize}
    \item for any $x,y\in\mathbb{R}^d$, as long as $\|x-y\|_2\le R$, then $\Pr_{h\in {\cal H}} [h(x)=h(y)]\ge p_1$;
    \item for any $x,y\in\mathbb{R}^d$, as long as $\|x-y\|_2\ge cR$, then $\Pr_{h\in {\cal H}} [h(x)=h(y)]\le p_2$.
\end{itemize}
\end{definition}
 
The approximate near neighbor problem is to find a point $p$ closed to the given point $q$ with some approximation error tolerance.  

\begin{definition}[Approximate Near Neighbor, \cite{ssx21}]
Let $c>1$ and $r\in (0,2)$. Given an $n$-point
dataset $P\subset \mathbb{S}^{d-1}$ on the sphere, the goal of the $(c, r)$-Approximate Near Neighbor (ANN) problem is to build a data structure that, given a query $q\in \mathbb{S}^{d-1}$ with the promise that there exists a datapoint $p\in P$ with $\|p-q\|_2\le r$ returns a datapoint $p'\in P$ within distance $c\cdot r$ from q.
\end{definition}
 
Implementing a naive algorithm in a brute force way, we can solve the approximate near neighbor problem in $O(n d)$ time and $O(n d)$ memory space.  
A faster solution can be given in time $O(d\cdot n^{\rho})$ in $O(d\cdot n^{\rho})$ time and $O(n^{1+\rho}+dn)$ space (\cite{ar15}).

\begin{theorem}[Andoni and Razenshteyn \cite{ar15} ]\label{thm:ANN}
Let $c>1$ and $r\in(0,2)$. Let $\rho=\frac{1}{2c^2-1}+o(1)$. The $(c,r)$-ANN on a unit sphere $\mathbb{S}^{d-1}$ can be solved in space $O(n^{1+\rho}+dn)$ and query time $O(d\cdot n^{\rho})$.
\end{theorem}

\subsection{Our Results}\label{sec:app_alg_lsh:result}
In this section, we give a greedy selection algorithm with data structures related to locality sensitive hashing (LSH). 
We first give a definition of asymmetric transformation, e.g. see \cite{ssx21}.
\begin{definition}[Asymmetric transformation
]
Let $B\subseteq \mathbb{R}^d$ and $\forall  b\in B,\|b\|_2\le 1$. Let $a\in \mathbb{R}^d$ with $\|a\|_2=1$. We define the following asymmetric transformation:
\begin{align*}
     P(b)=[b^{\top}~\sqrt{1-\|b\|_2^2}~0]^\top \in \R^{d+2}, Q(a)=[a^\top~0~\sqrt{1-\|a\|_2^2}]^\top \in \R^{d+2}
\end{align*}
Therefore, we have 
\begin{align*}
    \|Q(a)\|_2 = 1, \|P(b)\|_2 = 1, \langle Q(a), P(b) \rangle = \langle a, b \rangle.
\end{align*}
\end{definition}

\subsection{Details of Data Structure with LSH.}\label{sec:app_alg_lsh:details}

The detailed data structure for fast quadratic form search is as described in Algorithm~\ref{alg:app_lsh_fqfs}.

\begin{algorithm}[!ht]\caption{Fast Quadratic Form Search Data Structure with LSH} \label{alg:app_lsh_fqfs}
\begin{algorithmic}[1]
    \State {\bf data structure} \textsc{FQFSwithLSH}
    \State {\bf members}
    \State \hspace{4mm} $S$ (Candidate index set)
    \State \hspace{4mm} \textsc{Max-IP} $\text{maxip}$
    \State {\bf end members}
    \State 
    \Procedure{\textsc{Init}}{$\{u_1,u_2,\cdots,u_n\}\subset \mathbb{R}^d, \epsilon > 0, \delta > 0$}
        \State $S = [n]$
        \For {$i=1$ to $n$}
            \State $v_i \gets [ ~ (u_i)_1 \cdot u_i^\top ~~ (u_i)_2 \cdot u_i^\top ~~ \cdots ~~ (u_i)_d \cdot u_i^\top ~]^\top$
        \EndFor 
        \State $\text{maxip}$.\textsc{Init}$(\{Q(v_1),Q(v_2),\cdots,Q(v_n)\}, \epsilon, \delta)$
    \EndProcedure
    \State 
    \Procedure{\textsc{Query}}{$M\in \mathbb{R}^{d\times d}$}
        \State Write $M$ as $[ ~ m_1 ~ m_2 ~ \cdots ~ m_d ~ ]$
        \State $m \gets [ ~ m_1^\top ~ m_2^\top ~ \cdots ~ m_d^\top ~ ]^\top$
        \State $i \gets \text{maxip}$.\textsc{Query}$(P(m))$
        \State \Return $v_i$
    \EndProcedure
    \State 
    \Procedure{\textsc{Delete}}{$i\in[n]$}
        \State $S \gets S \backslash \{i\}$
        \State maxip.\textsc{Delete}$(i)$ 
    \EndProcedure
    \State {\bf end data structure}
\end{algorithmic}
\end{algorithm}

\begin{algorithm}[!ht]\caption{Fast submodular maximization algorithm using LSH}\label{alg:app_lsh_combination_alg}
\begin{algorithmic}[1]
\Procedure{FastSubmodularMaximizationLSH}{$n,k,h,\{u_1, \cdots, u_n \} \subset \R^d$}
    \State \textsc{FQFSL} ds.\textsc{Init}($u_1, \cdots u_n, \epsilon, \delta/k$) \Comment{Algorithm \ref{alg:app_lsh_fqfs}}
    \State $S_0 \gets \emptyset$
    \For{$t=0 \to k-1$}
        \State $A_t \gets h (S_t)$
        \State $j \gets \text{ds}.\textsc{Query}(A_t)$
        \State $S_{t+1} \gets S_t \cup \{ j \}$
        \State $\text{ds}.\textsc{Del}(j)$
     \EndFor
    \State \Return $S_k$
\EndProcedure
\end{algorithmic}
\end{algorithm}

\subsection{Proof of Corollary~\ref{cor:ds_LSH_submodular_problem}}\label{sec:app_alg_lsh:LSH_problem}
\begin{proof}

Similarly, it follows from Theorem~\ref{thm:main_formal_LSH} and a variant of Corollary~\ref{cor:error_approximate_greedy_algo}.
\end{proof}

\section{Intuitive Extension to Typical Constraints and Online Setting}\label{sec:extension}
In practice, there are many variants of the setting we showed in Section~\ref{sec:preli}.
For a general optimization problem, it is usually taxonomized based on three perspectives: (i) objective function, (ii) constraint, and (iii) decision time.
For the objective function, as we have discussed in Section~\ref{sec:preli}, our embedding representation is fairly general, as the embedding representation has enough degree of freedom.
For the constraint, in addition to cardinality constraint in the form of $|S|\le k$, there are constraints of other structures in applications, e.g., knapsack constraints and matroid constraints.
For the decision time, if the decision must be made before knowing the complete information of the objective function and the constraints, it is called online optimization.

For simplicity, we demonstrated our accelerating method on the specific case of cardinality constraint and offline decision making.
However, our method could be naturally extended to other constraints and online setting.
So we will discuss the extension in this section.
More detailed proofs related to this section is refered to Section~\ref{sec:app_extension}. In Section~\ref{sec:extension:matroid}, we discuss the matroid constraints, where the maximization of $f(S)$ is over all subsets in some matroid $\mathcal{I}$. In Section~\ref{sec:extension:knapsack}, we discuss the knapsack constraints, where for including each element $j$ in $S$ we need to pay cost $w(j) \ge 0$ and the total cost must be less than or equal to a given budget $W$. In Section~\ref{sec:extension:online}, we present the online decision.

\subsection{Constraints: Matroid Constraints}\label{sec:extension:matroid}

We introduce the definition of matroid in Definition~\ref{def:matroid} and then define the submodular function maximization subject to matroid constraints.
\begin{definition}[Matroid]\label{def:matroid}
The pair $M=(E,\mathcal{I})$ is called a matroid where $\mathcal{I}$ is a family of subsets of $E$, if
$\mathcal{I}$ contains empty set, i.e., $\emptyset \in \mathcal{I}$.
$\mathcal{I}$ is downward-closed, i.e., for each $A\subseteq B \subseteq E$, if $B \in \mathcal{I}$, then $A\in E$.
$\mathcal{I}$ satisfies augmentation property, i.e., if $A, B\in\mathcal{I}$, $|A| < |B|$, then there exsits $x\in B-A$ such that $A+\{x\}\in\mathcal{I}$.
\end{definition}
Suppose $M=([n], \mathcal{I})$ is a matroid, then we call the optimization problem \eqref{eqn:matroid_constraint} the submodular maximization with matroid constraints.
\begin{equation}\label{eqn:matroid_constraint}
\begin{aligned}
\max_{S\in \subseteq[n]} \quad & f(S)\\
\textrm{s.t.} \quad & S \in \mathcal{I}
\end{aligned}
\end{equation}
Fortunately, the greedy selection algorithm can be adjusted to be still effective for the matroid constraint.
Now we greedily update $S_t$ as $S_{t+1}\gets S_t + \{j_t\}$, where $j_t$ is selected as
\begin{align}\label{eqn:greedy_selection_matroid}
    j_t = \arg \max_{\{j\}+S_t\in\mathcal{I},\; j\in\ov{S}_t} \Delta_f(j|S_t).
\end{align}
This greedy selection guarantees a $1/2$ approximation error constant \cite{nwf78}.
Note that $j_t$ in \eqref{eqn:greedy_selection_matroid} shares similar form as greedy selection for cardinality constraint, except that we further restrict $j$ to satisfy $\{j\}+S_t$.
Algorithm~\ref{alg:app_fast_submodular_matroid} formally states how to apply our method. The following lemmas and theorems present the correctness and time complexity. 

\begin{lemma}[Correctness]\label{lem:error_approximate_greedy_algo_matroid}
Given $\epsilon > 0$. Suppose $M=([n],{\cal I})$ is a matroid. Consider an approximate greedy algorithm $A$ (see Algorithm \ref{alg:app_fast_submodular_matroid}) for the submodular function maximization with matroid constraint $M$ that has an oracle $O$: taken an arbitrary set $S\subseteq[n]$ and $i \in [n]\backslash S$, returns a value $O(S,i)$ with guarantee that $\Delta(i|S)-\epsilon \le O(S,i) \le \Delta(i|S)+\epsilon$. And in each time step $t = 1,2,\cdots,\rank(M)$, it selects 
\begin{align*}
   j_t = \operatornamewithlimits{argmax}_{\{j\}+S_t\in\mathcal{I},\; j\in\ov{S}_t} \{O(S_{t-1},j)\},
\end{align*}
and lets $S_t \gets S_{t-1} \cup \{j_t\}$, then this algorithm makes sure that it returns a set $S_{\rank(M)}$ with
\begin{align*}
    f(S_{\rank(M)}) \ge \frac{1}{2} \max_{T\in\mathcal{I}}\{f(T)\} - \rank(M)\epsilon.
\end{align*}
\end{lemma}

The proof is in Section~\ref{sec:app_extension}.
\begin{theorem} \label{thm:main_formal_matroid}
There is a submodular function maximization algorithm such that, given a submodular function maximization problem with matroid constraint and two precision parameters $\epsilon,\delta$ as input, it runs in $O(\epsilon^{-2}(nd+n^2+ \rank(M)d^2)\log(n/\delta))$ time and returns a solution $S$ with $f(S) \ge 1/2\max_{S\in\mathcal{I}}\{f(S)\} - \rank(M)\epsilon$  with probability at least $1-\delta$.
\end{theorem}

The proof is in Section~\ref{sec:app_extension}.

\subsection{Constraints: Knapsack Constraints}\label{sec:extension:knapsack}
 
Another natural constraint that often appears in the submodular function maximization problem is the knapsack constraint.
It views the ground set as indices of items and each item $j\in[n]$ has a weight $w(j)$.
We have a knapsack to put some items into it, subject to the constraint that the total weight of selected items can not exceed the maximum load of the knapsack.

\begin{definition}[Maximization with knapsack constraints]
We call the maximization problem $\max_{S \subseteq [n]}  f(S)~~~\textrm{s.t.}\sum_{j\in S} w(j) \le W$ as the submodular function maximization problem with knapsack constraints.
where $f$ is the objective submodular function, $w(j)\ge 0$ is the weight for item $j$, and $W\ge 0$ is the maximum load of the knapsack.
\end{definition}

Similarly, a natural way to adjust the greedy selection algorithm is to update $S_t$ as 
 $
S_{t+1}\gets S_t + \{j_t\},
 $
where $j_t$ is selected as
\begin{align}\label{eqn:greedy_selection_knapsack}
   j_t = \operatornamewithlimits{argmax}_{w(j)+w(S_t) \le W,\; j\in\ov{S}_t} \frac{\Delta_f(j|S_t)}{w(j)}.
\end{align}
Compare the updating rule \eqref{eqn:greedy_selection_knapsack} for the knapsack case to the rule in Algorithm~\ref{alg:greedy}, we note that there are two differences: in \eqref{eqn:greedy_selection_knapsack} we further restrict $j$ to satisfy $w(j)+w(S_t) \le W$ and we substitute $\Delta_f(j|S_t)$ by $\frac{\Delta_f(j|S_t)}{w(j)}$. For the additional restriction
$
w(j)+w(S_t) \le W,
$
we can add this restriction to our data structure as well.
For the $\Delta_f(j|S_t)$, we can modify our embedding representation.
Let $v_j=\frac{u_j}{\sqrt{w(j)}}$, then we have 
\begin{align*}
\frac{\Delta_f(j|S)}{w(j)} 
= \frac{u_j^\top h(S) u_j}{w(j)} =  v_j^\top h(S) v_j.
\end{align*}
The search for $j$ that maximizes $v_j^\top h(S) v_j$ is again in the form of queries can be solved by our dynamic data structure.However, simply using the greedy strategy \eqref{eqn:greedy_selection_knapsack} is not guaranteed to yield a small error.
Fortunately, it is proved that greedy selection by uniform cost and greedy selection by benefit per cost can not be bad at the same time \cite{leskovec2007cost}.
Combing the two greedy strategies, an error ratio of $1/2-1/(2e)$ is guaranteed.
Algorithm~\ref{alg:app_fast_submodular_knapsack} formally describes how to extend our method to the knapsack constraints. The following lemmas and theorems present the correctness and time complexity.

\begin{lemma}[Correctness] \label{lem:error_approximate_greedy_algo_knapsack}
Given $\epsilon > 0$. Consider an approximate greedy algorithm $A$ (see Algorithm \ref{alg:app_fast_submodular_knapsack}) for the submodular function maximization with knapsack constraints that has an oracle $O$: taken an arbitrary set $S\subseteq[n]$ and $i \in [n]\backslash S$, returns a value $O(S,i)$ with guarantee that $\Delta(i|S)-\epsilon \le O(S,i) \le \Delta(i|S)+\epsilon$ and $O'(S,i)$ with guarantee that 
$
\frac{\Delta(i|S)}{w(i)}-\epsilon \le O'(S,i) \le \frac{\Delta(i|S)}{w(i)}+\epsilon
$.

In the first pass, in each time step $t$, it selects 
$
    j_t = \operatornamewithlimits{argmax}_{w(\{j\}+S_t) \le W,\; j\in\ov{S}_t} O(S_{t-1},j),
$
and lets $S_t \gets S_{t-1} \cup \{j_t\}$.
Suppose the first pass ends with a set $S_k$.
In the second pass, in each time step $t$, it selects
 $
   j'_t = \operatornamewithlimits{argmax}_{w(\{j\}+S_t) \le W,\; j\in\ov{S}_t} O'(S_{t-1},j),
 $
and lets 
 $
S'_t \gets S'_{t-1}\cup\{j'_t\}
 $.
Suppose the second pass ends with a set $S'_{k'}$.
Then it holds that
 $
    \max\{f(S_k), f(S'_{k'})\} \ge (1/2-1/(2e)) \max_{T:w(T) \le W}\{f(T)\} - L\epsilon.
 $
\end{lemma}

The proof is in Section~\ref{sec:app_extension}.

\begin{theorem} \label{thm:main_formal_knapsack}
There is a submodular function maximization algorithm such that, given a submodular function maximization problem with knapsack constraint and two precision parameters $\epsilon,\delta$ as input, it runs in $O(\epsilon^{-2}(nd+n^2+kd^2)\log(k/\delta))$ time, where $k$ is the size of the optimal solution, and returns a solution $S$ with
 $
    f(S) \ge (1/2-1/(2e))\max_{S:w(S) \le W}\{f(S)\} - L\epsilon
 $
with probability at least $1-\delta$.
\end{theorem}

The proof is in Section~\ref{sec:app_extension}.

\subsection{Online Decision}\label{sec:extension:online}

In the online decision setting, we receive one new element of the ground set in each iteration $t$, and must decide $S_t$ immediately.
\begin{definition}[Online Setting in \cite{streeter2008online}]
Let $A$ denote a finite set of activities and 
$(a,\tau)\in A\times\mathbb{R}_{+}$ denote an action, which means spending time $\tau$ on activity $a$.
We define a schedule $s$ as a sequence of actions $s:=\langle (a_1,\tau_1), \cdots (a_m, \tau_m)\rangle$.
Let $\mathcal{S}$ denote the set of all schedules.
Let $f^{(1)}\cdots f^{(n)}$ denote $n$ jobs,
where job $f^{(i)}:\mathcal{S}\rightarrow [0,1]$ is a monotone submodular function.
In each iteration, one job $f^{(i)}$ arrives.
At time $t$, we have the information about $\{f^{(i)}\}_{i=1}^{t-1}$.
The schedule $S^{(i)}$ for $f^{(i)}$ must be at time $t$.
We define the time cost by a schedule $S$ as $\ell(S):=\sum_{(a,\tau)\in S} \tau$.
The goal of online submodular function maximization is to minimize the regret $R$, which is defined as
 $
    R(S^{(1)},\cdots,S^{(n)}) := (1-1/e)\frac{\sum_{i=1}^n f^{(i)}(S^*)}{n} - \frac{\sum_{i=1}^n f^{(i)}(S^{(i)})}{n}.
 $
where 
 $
S^*:=\arg\max_{S\in\mathcal{S}} \frac{\sum_{i=1}^n f^{(i)}(S*)}{n}
 $
is the optimal schedule, subject to the constraint 
 $
\E[\ell(S^{(i)})] = T
 $.
\end{definition}

An algorithm named online greedy (OG) achieving small regret bound is proposed in \cite{streeter2008online} for the online submodular function maximization.
For the OG algorithm, from the perspective of operations, the main bottleneck is still greedily searching $\frac{\Delta_f(j|S_{t-1})}{\tau_j}$.
Hence our method could be adjusted for the online setting too.
\section{ Extension to Typical Constraints and Online Setting}\label{sec:app_extension}

In this section, we demonstrate how to extend our method to accelerate variants of submodular maximization problems, e.g., submodular maximization subject to different constraints and online submodular function maximization. In Section~\ref{sec:app_extension:matroid}, we discuss the matroid constraints, where the maximization of $f(S)$ is over all subsets in some matroid $\mathcal{I}$. In Section~\ref{sec:app_extension:knapsack}, we discuss the knapsack constraints, where for including each element $j$ in $S$ we need to pay cost $w(j) \ge 0$ and the total cost must be less than or equal to a given budget $W$. In Section~\ref{sec:app_extension:online}, we present the online decision.

\subsection{Constraints: Matroid Constraints}\label{sec:app_extension:matroid}

We introduce the definition of matroid in Definition~\ref{def:app_matroid} and then define the submodular function maximization subject to matroid constraints.
\begin{definition}[Matroid]\label{def:app_matroid}
The pair $M=(E,\mathcal{I})$ is called a matroid where $\mathcal{I}$ is a family of subsets of $E$, if
\begin{itemize}
    \item $\mathcal{I}$ contains empty set, i.e., $\emptyset \in \mathcal{I}$.
    \item $\mathcal{I}$ is downward-closed, i.e., for each $A\subseteq B \subseteq E$, if $B \in \mathcal{I}$, then $A\in E$.
    \item $\mathcal{I}$ satisfies augmentation property, i.e., if $A, B\in\mathcal{I}$, $|A| < |B|$, then there exsits $x\in B-A$ such that $A+\{x\}\in\mathcal{I}$.
\end{itemize}
\end{definition}
Suppose $M=([n], \mathcal{I})$ is a matroid, then we call the optimization problem \eqref{eqn:app_matroid_constraint} the submodular maximization with matroid constraints.
\begin{equation}\label{eqn:app_matroid_constraint}
\begin{aligned}
\max_{S\in \subseteq[n]} \quad & f(S)\\
\textrm{s.t.} \quad & S \in \mathcal{I}
\end{aligned}
\end{equation}
Fortunately, the greedy selection algorithm can be adjusted to be still effective for the matroid constraint.
Now we greedily update $S_t$ as $S_{t+1}\gets S_t + \{j_t\}$, where $j_t$ is selected as
\begin{align}\label{eqn:app_greedy_selection_matroid}
    j_t = \arg\max_{\{j\}+S_t \in \mathcal{I},\; j \in \ov{S}_t} \Delta_f(j \mid S_t).
\end{align}
This greedy selection guarantees a $1/2$ approximation error constant \cite{nwf78}.
Note that $j_t$ in \eqref{eqn:app_greedy_selection_matroid} shares similar form as greedy selection for cardinality constraint, except that we further restrict $j$ to satisfy $\{j\}+S_t$.
Algorithm~\ref{alg:app_fast_submodular_matroid} formally states how to apply our method.

\begin{algorithm}[!ht]\caption{Fast Submodular Function Maximization with Matroid Constraint }\label{alg:app_fast_submodular_matroid}
\begin{algorithmic}[1]
\Procedure{FastMatroidSubmodular}{$n,M=([n],\mathcal{I}),\{u_1, \cdots, u_n \} \subseteq \R^d$}
    \State $S_0 \gets \emptyset$
    \State FQFS ds.\textsc{Init}($u_1,u_2,\cdots,u_n,\epsilon,\delta/n$)
    \For{$t=1 \to \rank(M)$}
        \State ds.\textsc{Delete}($j$) for $S_{t-1}+\{j\} \notin \mathcal{I}$ \label{line:app_matroid_alg_delete1}
        \State $A_{t-1} \gets h (S_{t-1})$
        \State $j \gets $ ds.\textsc{Query}($A_{t-1}$)
        \State $S_{t} \gets S_{t-1} \cup \{j\}$
        \State ds.\textsc{Delete}($j$) \label{line:app_matroid_alg_delete2}
    \EndFor
    \State \Return $S_{\rank(M)}$
\EndProcedure
\end{algorithmic}
\end{algorithm}

\begin{lemma}[Correctness,restatement of Lemma~\ref{lem:error_approximate_greedy_algo_matroid}]\label{lem:app_error_approximate_greedy_algo_matroid}
Given $\epsilon > 0$. Suppose $M=([n],{\cal I})$ is a matroid. Consider an approximate greedy algorithm $A$ (see Algorithm \ref{alg:app_fast_submodular_matroid}) for the submodular function maximization with matroid constraint $M$ that has an oracle $O$: taken an arbitrary set $S\subseteq[n]$ and $i \in [n]\backslash S$, returns a value $O(S,i)$ with guarantee that $\Delta(i|S)-\epsilon \le O(S,i) \le \Delta(i|S)+\epsilon$. And in each time step $t = 1,2,\cdots,\rank(M)$, it selects 
\begin{align*}
    j_t = \operatornamewithlimits{argmax}_{\{j\}+S_t \in \mathcal{I},\; j \in \ov{S}_t} O(S_{t-1},j),
\end{align*}
and lets $S_t \gets S_{t-1} \cup \{j_t\}$, then this algorithm makes sure that it returns a set $S_{\rank(M)}$ with
\begin{align*}
    f(S_{\rank(M)}) \ge \frac{1}{2} \max_{T\in\mathcal{I}}\{f(T)\} - \rank(M)\epsilon.
\end{align*}
\end{lemma}

\begin{proof}

For convenience, let $k = \rank(M)$. For $t = 1,2,\cdots,k$, define $\Delta_{t-1}$ be the increment of the greedy algorithm to the objective function $f$ at time step $t$. Let $T$ and $S$ be optimal and greedy solutions respectively, with $|T| = |S| = k$.

For $t = 0, 1, \cdots, k$, let $U_{t-1}$ be the set of all elements considered during the first $t$ iterations before the addition of a $t$-th element to $S_{t-1}$, that is, the set of the deleted indices (line \ref{line:app_matroid_alg_delete1} and line \ref{line:app_matroid_alg_delete2} in Algorithm \ref{alg:app_fast_submodular_matroid}) at time step $t=0,1,\cdots,t$. For $t = 1,\cdots, k$, let $s_{t-1} = |T \cap (U_t - U_{t-1}) |$, which describes whether the element added to $S$ by the greedy algorithm at time step $t$ belongs to $T$. 

For a set $S \subseteq [n]$ (not necessary in $M$), define $\text{rank}(S)$ be the cardinality of the largest subset of $S$ in $M$, and define $span(S) = \{ j \in N | \text{rank}(S\cup\{j\}) = \text{rank}(S) \}$.

First we show 
\begin{align}
    \sum_{j\in T-S}\Delta(j|S) \le \sum_{i=1}^{k} (\Delta_{i-1}+2\epsilon)s_{i-1}. \label{eq:matroid_constraint1}
\end{align}
This is because 
\begin{align*}
    \sum_{j\in T-S} \Delta(j|S) 
    \le & ~ \sum_{j \in T} \Delta(j|S) \\
    = & ~ \sum_{t=1}^K \sum_{j \in T\cap (U^t-U^{t-1})} \Delta(j|S) \\
    \le & ~ \sum_{t=1}^K (\Delta_{t-1}+2\epsilon)s_{t-1},
\end{align*} 
where the first step follows from $f$ is nondecreasing, the second step follows from the partition $T=T\cap\sum_{t=1}^{k} (U^t-U^{t-1})$ and the last step follows from the greedy rule and the oracle guarantee.

Second we show that $U^t \subseteq \text{span}(S^t)$. This is because for an element $x$ of $U^t$, if $x\notin \text{span}(S^t)$, then $x\notin S^t$, implying $x$ is excluded in line \ref{line:app_matroid_alg_delete1} of Algorithm \ref{alg:app_fast_submodular_matroid} at a time step no larger than $t$, which means $S^t\cup\{x\} \notin M$, then $x \in \text{span}(S^t)$, a contradiction. 

Third we show $\sum_{i=1}^t s_{i-1} \le t$ for $t = 1, \cdots, k$. This is because $\sum_{i=1}^t s_{i-1} = |T \cap U^t|$. Since $U^t \subseteq \text{span}(S^t)$ we have $|T \cap U^t| \le |T \cap \text{span}(S^t)|$. But since $T \in M$ and $\rank(\text{span}(S^t)) = |S^t| = t$, we obtain $|T \cap \text{span}(S^t)| \le t$. Therefore $\sum_{i=1}^t s_{i-1} \le t$.

Combining them,  
we have 
\begin{align}
    \sum_{i=1}^{k}\Delta_{i-1}s_{i-1} = & ~ \sum_{i=1}^{k-1}(\Delta_{i-1} - \Delta_i) \sum_{j=1}^i s_{j-1} + \Delta_{k-1}\sum_{j=1}^{k}s_{j-1} \notag \\
    \le & ~ \sum_{i=1}^{k-1} (\Delta_{i-1} - \Delta_i)i + \Delta_{k-1} (k-1) \notag \\
    = & ~ \sum_{i=1}^k \Delta_{i-1} \notag \\
    = & ~ f(S), \label{eq:matroid_constraint2} 
\end{align}
where the first step in the proof utilizes the method of summation by parts, the second step is derived from the inequality $\sum_{i=1}^ts_{i-1} \le t$, the third step is obtained by consolidating these terms. Lastly, the final step is justified by the equality $\sum_{i=1}^k\Delta_{i-1} = f(S)$.

Then we have
\begin{align*}
    f(T) \le & ~ f(S) + \sum_{j\in T-S}\Delta(j|S) \\
    \le & ~ f(S) + f(S) + 2\epsilon \sum_{i=1}^k s_{i-1} \\
    = & ~ 2f(S) + 2\epsilon \rank(M),
\end{align*}
where the first step follows from properties of submodular functions, the second step follows from Eq. \eqref{eq:matroid_constraint1} and \eqref{eq:matroid_constraint2}, and the last step follows from $\sum_{i=1}^k s_{i-1} = k = \rank(M)$.

Hence $f(S) \ge \frac{1}{2}f(T) - \rank(M)\epsilon$, which finishes the proof.
\end{proof}

\begin{theorem}[Restatement of Theorem~\ref{thm:main_formal_matroid}] \label{thm:app_main_formal_matroid}
There is a submodular function maximization algorithm such that, given a submodular function maximization problem with matroid constraint and two precision parameters $\epsilon,\delta$ as input, it runs in $O(\epsilon^{-2}(nd+n^2+ \rank(M)d^2)\log(n/\delta))$ time and returns a solution $S$ with $f(S) \ge 1/2\max_{S\in\mathcal{I}}\{f(S)\} - \rank(M)\epsilon$  with probability at least $1-\delta$.
\end{theorem}
\begin{proof}
Combining Lemma~\ref{lem:app_error_approximate_greedy_algo_matroid} (the correctness of Algorithm~\ref{alg:app_fast_submodular_matroid}) and time complexity analysis of operations in Section~\ref{sec:app_data_structure}, we prove the conclusion.
\end{proof}

\subsection{Constraints: Knapsack Constraints}\label{sec:app_extension:knapsack}
 
Another natural constraint that often appears in the submodular function maximization problem is the knapsack constraint.
It views the ground set as indices of items and each item $j\in[n]$ has a weight $w(j)$.
We have a knapsack to put some items into it, subject to the constraint that the total weight of selected items can not exceed the maximum load of the knapsack.

\begin{definition}[Maximization with knapsack constraints]
We call the maximization problem \eqref{eqn:app_knapsack_constraint} as the submodular function maximization problem with knapsack constraints.
\begin{equation}\label{eqn:app_knapsack_constraint}
\begin{aligned}
\max_{S \subseteq [n]} \quad & f(S)\\
\textrm{s.t.} \quad & \sum_{j\in S} w(j) \le W
\end{aligned}
\end{equation}
where $f$ is the objective submodular function, $w(j)\ge 0$ is the weight for item $j$, and $W\ge 0$ is the maximum load of the knapsack.
\end{definition}

Similarly, a natural way to adjust the greedy selection algorithm is to update $S_t$ as 
\begin{align*}
S_{t+1}\gets S_t + \{j_t\},
\end{align*}
where $j_t$ is selected as
\begin{align}\label{eqn:app_greedy_selection_knapsack}
    j_t = \operatornamewithlimits{argmax}_{w(j)+w(S_t) \le W,\; j \in \ov{S}_t} \frac{\Delta_f(j \mid S_t)}{w(j)}.
\end{align}
Compare the updating rule \eqref{eqn:app_greedy_selection_knapsack} for the knapsack case to the rule in Algorithm~\ref{alg:greedy}, we note that there are two differences: in \eqref{eqn:app_greedy_selection_knapsack} we further restrict $j$ to satisfy $w(j)+w(S_t) \le W$ and we substitute $\Delta_f(j|S_t)$ by $\frac{\Delta_f(j|S_t)}{w(j)}$.

For the additional restriction
\begin{align*}
w(j)+w(S_t) \le W,
\end{align*}
we can add this restriction to our data structure too.
For the $\Delta_f(j|S_t)$, we can modify our embedding representation.
Let $v_j=\frac{u_j}{\sqrt{w(j)}}$, then we have 
\begin{align*}
\frac{\Delta_f(j|S)}{w(j)} 
= & ~ \frac{u_j^\top h(S) u_j}{w(j)} \\
= & ~ v_j^\top h(S) v_j.
\end{align*}
The search for $j$ that maximizes $v_j^\top h(S) v_j$ is again in the form of queries can be solved by our dynamic data structure.

However, simply using the greedy strategy \eqref{eqn:app_greedy_selection_knapsack} is not guranteed to yield a small error.
Fortunately, it is proved that greedy selection by uniform cost and greedy selection by benefit per cost can not be bad at the same time \cite{leskovec2007cost}.
Combing the two greedy strategies, an error ratio of $1/2-1/(2e)$ is guaranteed.
Algorithm~\ref{alg:app_fast_submodular_knapsack} formally describes how to extend our method to the knapsack constraints.

\begin{algorithm}[!ht]\caption{Fast Submodular Function Maximization Algorithm with Knapsack Constraint}\label{alg:app_fast_submodular_knapsack}
\begin{algorithmic}[1]
\Procedure{FastKnapsackSelection}{$n,W,\{w(1),\cdots,w(n)\},\{u_1, \cdots, u_n \}\subseteq \R^d$}
    \State $t \gets 0$
    \State $S_0 \gets \emptyset$
    \State FQFS ds.\textsc{Init}($u_1,u_2,\cdots,u_n,\epsilon,\delta/n$) 
    \While{ds is not empty}
        \State $A_t \gets h (S_t)$
        \State $j \gets $ ds.\textsc{Query}($A_t$)
        \If{$w(S_t) + w(j) \le W$}
            \State $S_{t+1} \gets S_t \cup \{j\}$
            \State $t \gets t+1$
        \EndIf
        \State ds.\textsc{Delete}($j$)
    \EndWhile
    \State \Return $S_t$
\EndProcedure
\State 
\Procedure{FastKnapsackSubmodular}{$n,W,\{w(1),\cdots,w(n)\},\{u_1, \cdots, u_n \}\subseteq \R^d$}
    \For{$i = 1 \to n$}
        \State $v_i \gets u_i / \sqrt{w(i)}$
    \EndFor
    \State $S_{\text{uc}} \gets \text{FastKnapsackSelection}(n, W, \{w(i)\}_{i=1}^n, \{u_i\}_{i=1}^n)$
    \State $S_{\text{cb}} \gets \text{FastKnapsackSelection}(n, W, \{w(i)\}_{i=1}^n, \{v_i\}_{i=1}^n)$
    \If{$f(S_{\text{uc}}) > f(S_{\text{cb}})$}
        \State $S \gets S_{\text{uc}}$
    \Else
        \State $S \gets S_{\text{cb}}$
    \EndIf
    \State \Return $S$
\EndProcedure
\end{algorithmic}
\end{algorithm}

\begin{lemma}[Correctness,restatement of Lemma~\ref{lem:error_approximate_greedy_algo_knapsack}] \label{lem:app_error_approximate_greedy_algo_knapsack}
Given $\epsilon > 0$. Consider an approximate greedy algorithm $A$ (see Algorithm \ref{alg:app_fast_submodular_knapsack}) for the submodular function maximization with knapsack constraints that has an oracle $O$: taken an arbitrary set $S\subseteq[n]$ and $i \in [n]\backslash S$, returns a value $O(S,i)$ with guarantee that $\Delta(i|S)-\epsilon \le O(S,i) \le \Delta(i|S)+\epsilon$ and $O'(S,i)$ with guarantee that $\frac{\Delta(i|S)}{w(i)}-\epsilon \le O'(S,i) \le \frac{\Delta(i|S)}{w(i)}+\epsilon$.
In the first pass, in each time step $t$, it selects 
\begin{align*}
    j_t = \operatornamewithlimits{argmax}_{w(\{j\}+S_t) \le W,\; j \in \ov{S}_t} O(S_{t-1},j),
\end{align*}
and lets $S_t \gets S_{t-1} \cup \{j_t\}$.
Suppose the first pass ends with a set $S_k$.
In the second pass, in each time step $t$, it selects
\begin{align*}
    j'_t = \operatornamewithlimits{argmax}_{w(\{j\}+S_t) \le W,\; j \in \ov{S}_t} O'(S_{t-1},j),
\end{align*}
and lets $S'_t \gets S'_{t-1}\cup\{j'_t\}$.
Suppose the second pass ends with a set $S'_{k'}$.
Then it holds that
\begin{align*}
    \max\{f(S_k), f(S'_{k'})\} \ge (1/2-1/(2e)) \max_{T:w(T) \le W}\{f(T)\} - L\epsilon.
\end{align*}
\end{lemma}
\begin{proof}

Consider the second pass following Eq. \eqref{eqn:app_greedy_selection_knapsack} (corresponding to $S_{uc}$ in Algorithm \ref{alg:app_fast_submodular_knapsack}). Renumber the ground set $S = \{x_1,\cdots,x_n\}$ such that for each $i=0,1,\cdots,n$, the set of the selected elements at time step $i$ is $G_i := \{x_1,\cdots,x_i\}$ ($G_0:=\emptyset$). Let 
\begin{align*}
l = \arg\max_{i : c(G_i) \le B}\{c(G_i)\}
\end{align*}
be the index corresponding to the iterations. Hence $G_l$ is the selected set by the weighted greedy process. 

Let $\OPT$ be the optimal set, and let $L = c(\OPT)$, $w = |\OPT|$.

By greedy rule, 
\begin{align*}
    \frac{\Delta(x_i|G_{i-1})}{c(x_i)} + \epsilon \ge & ~ O'(G_{i-1},x_i) \\
    \ge & ~ \max_{y\in \ov{G_i}} \{O'(G_{i-1},y)\} \\
    \ge & ~ \max_{y \in \ov{G_i}} \{\frac{\Delta(y|G_{i-1})}{c(y)}\} - \epsilon.
\end{align*}

Since $f$ is nondecreasing, for $i=1,2,\cdots,l+1$, we have 
\begin{align*}
    f(\OPT) - f(G_{i-1}) &~ \le f(\OPT \cup G_{i-1}) - f(G_{i-1}) \\
    &= f(\OPT\backslash G_{i-1} \cup G_{i-1}) - f(G_{i-1}).
\end{align*}
Assume $\OPT \backslash G_{i-1} = \{y_1,\cdots,y_m\}$, and for $j=0,1,\cdots,m$, let $Z_j = f(G_{i-1}\cup\{y_1,y_2,\cdots,y_j\})$, then 
\begin{align*}
f(\OPT) - f(G_{i-1}) \le \sum_{j=1}^m \Delta(y_j | Z_{j-1}).
\end{align*}

Now notice that 
\begin{align*}
    \frac{\Delta(y_j | Z_{j-1})}{c(y_j)} - \epsilon \le \frac{\Delta(y_j|G_{i-1})}{c(y_j)} - \epsilon \le \frac{f(G_i) - f(G_{i-1})}{c(x_i)} + \epsilon,
\end{align*}
where the first step follows from $f$ is nondecreasing and the second step follows from the greedy rule, we have 
\begin{align*}
    \Delta(y_j|Z_{j-1}) \le \frac{c(y_j)}{c(x_i)}(f(G_i)-f(G_{i-1})) + 2 c(y_j) \epsilon.
\end{align*}

Since $\sum_{j=1}^m c(y_j) \le L$, it holds that 
\begin{align*}
    f(\OPT) - f(G_{i-1}) = & ~ \sum_{j=1}^m \Delta(y_j | Z_{j-1}) \\
    \le & ~ L \frac{f(G_{i}) - f(G_{i-1})}{c(x_i)} + 2L\epsilon,
\end{align*}
that is, 
\begin{align*}
    f(G_i) - f(G_{i-1}) \ge \frac{c(x_i)}{L}(f(\OPT) - f(G_{i-1})) - 2\epsilon.
\end{align*}

Then by induction, we have 
\begin{align*}
    f(G_i) \ge (1 - \prod_{k=1}^i(1-\frac{c(x_k)}{L})) f(\OPT) - 2L \epsilon.
\end{align*}
This is because first it holds for $i=1$; then suppose it holds for $i-1$, we have  
\begin{align}
    f(G_i) &\ge \frac{c(x_i)}{L}f(\OPT) + (1-\frac{c(x_i)}{L})f(G_{i-1}) - 2\epsilon \notag \\
    &\ge (1 - \prod_{k=1}^i(1-\frac{c(x_k)}{L})) f(\OPT) - 2L(1-\frac{c(x_i)}{L})\epsilon - 2\epsilon \notag\\
    &\ge (1 - \prod_{k=1}^i(1-\frac{c(x_k)}{L})) f(\OPT) - 2L(1-\frac{1}{L})\epsilon - 2\epsilon \notag \\
    &= (1 - \prod_{k=1}^i(1-\frac{c(x_k)}{L})) f(\OPT) - 2L \epsilon. \label{eq:app_knapsack_1}
\end{align}

After that, we have
\begin{align*}
    f(G_{l+1}) &\ge (1 - \prod_{k=1}^{l+1} (1 - \frac{c(x_k)}{L})) f(\OPT) - 2L \epsilon \\
    &\ge (1 - \prod_{k=1}^{l+1} (1 - \frac{c(x_k)}{G_{l+1}})) f(\OPT) - 2L \epsilon \\
    &\ge (1 - \prod_{k=1}^{l+1} (1 - \frac{1}{l+1}))  f(\OPT) - 2L \epsilon \\
    &\ge (1-\frac{1}{e})f(\OPT) - 2L \epsilon,
\end{align*}
where the first step follows from Eq. \eqref{eq:app_knapsack_1}, the second step follows from Jensen's inequality, 

Furthermore, note that the violating increase $F(G_{l+1}) - F(G_l)$ is bounded by $F(x^*)$ for $x^* = \arg\max_{x\in W} F(x)$, i.e. the second candidate solution considered by the modified greedy algorithm. Hence 
\begin{align*}
    f(G_l) + f(x^*) \ge & ~ f(G_{l+1} ) \\
    \ge & ~ (1-\frac{1}{e}) f(\OPT) - 2L \epsilon,
\end{align*}
so at least one of the values $F(x^*)$ or $F(G_l)$ must be greater than or equal to 
\begin{align*}
((1/2 - 1/(2e))f(\OPT) - L\epsilon).
\end{align*}
\end{proof}

\begin{theorem}[Restatement of Theorem~\ref{thm:main_formal_knapsack}] \label{thm:app_main_formal_knapsack}
There is a submodular function maximization algorithm such that, given a submodular function maximization problem with knapsack constraint and two precision parameters $\epsilon,\delta$ as input, it runs in $O(\epsilon^{-2}(nd+n^2+kd^2)\log(k/\delta))$ time, where $k$ is the size of the optimal solution, and returns a solution $S$ with
\begin{align*}
    f(S) \ge (1/2-1/(2e))\max_{S:w(S) \le W}\{f(S)\} - L\epsilon
\end{align*}
with probability at least $1-\delta$.
\end{theorem}
\begin{proof}
Combining Lemma~\ref{lem:app_error_approximate_greedy_algo_knapsack} (the correctness of Algorithm~\ref{alg:app_fast_submodular_knapsack}) and time complexity analysis of operations in Section~\ref{sec:app_data_structure}, we prove the conclusion.

\end{proof}

\subsection{Online Decision}\label{sec:app_extension:online}

In the online decision setting, we receive one new element of the ground set in each iteration $t$, and must decide $S_t$ immediately.
\begin{definition}[Online Setting in \cite{streeter2008online}]
Let $A$ denote a finite set of activities and 
$(a,\tau)\in A\times\mathbb{R}_{+}$ denote an action, which means spending time $\tau$ on activity $a$.
We define a schedule $s$ as a sequence of actions $s:=\langle (a_1,\tau_1), \cdots (a_m, \tau_m)\rangle$.
Let $\mathcal{S}$ denote the set of all schedules.
Let $f^{(1)}\cdots f^{(n)}$ denote $n$ jobs,
where job $f^{(i)}:\mathcal{S}\rightarrow [0,1]$ is a monotone submodular function.
In each iteration, one job $f^{(i)}$ arrives.
At time $t$, we have the information about $\{f^{(i)}\}_{i=1}^{t-1}$.
The schedule $S^{(i)}$ for $f^{(i)}$ must be at time $t$.
We define the time cost by a schedule $S$ as $\ell(S):=\sum_{(a,\tau)\in S} \tau$.
The goal of online submodular function maximization is to minimize the regret $R$, which is defined as
\begin{align*}
    R(S^{(1)},\cdots,S^{(n)}) := (1-1/e)\frac{\sum_{i=1}^n f^{(i)}(S^*)}{n} - \frac{\sum_{i=1}^n f^{(i)}(S^{(i)})}{n}.
\end{align*} 
where $S^*:=\arg\max_{S\in\mathcal{S}} \frac{\sum_{i=1}^n f^{(i)}(S*)}{n}$ is the optimal schedule, subject to the constraint  $\E[\ell(S^{(i)})] = T$.
\end{definition}

An algorithm named online greedy (OG) achieving small regret bound is proposed in \cite{streeter2008online} for the online submodular function maximization.
For the OG algorithm, from the perspective of operations, the main bottleneck is still greedily searching $\frac{\Delta_f(j|S_{t-1})}{\tau_j}$.
Hence our method could be adjusted for the online setting too.

\begin{algorithm}[!ht]\caption{Our conceptual algorithm -- semi-online }\label{alg:ours_online}
\begin{algorithmic}[1]
\Procedure{FastSubmodularMaximization}{$n,k,\{u_1, \cdots, u_n \} \subset \R^d$}
    \State \text{ds}.\textsc{Init}($u_1, \cdots u_n$)
    \State $S_0 \gets \emptyset$
    \For{$t=0 \to k-1$}
        \State $A_t \gets h (S_t)$
        \State Adversary gives a pair $(i,z)$ where $i \in [n]\backslash S_t$ and $z \in \R^d$
        \State $\text{ds}.\textsc{Update}(i,z)$ 
        \State $j \gets \text{ds}.\textsc{Query}(A_t)$
        \State $S_{t+1} \gets S_t \cup \{ j \}$
        \State $\text{ds}.\textsc{Del}(j)$
     \EndFor
    \State \Return $S_k$
\EndProcedure
\end{algorithmic}
\end{algorithm}





\section{Fast Submodular Function Maximization with Data Structure} \label{sec:combination}


In this section, we present the missing proof of Theorem~\ref{thm:main_formal}.

\begin{theorem}[Combination of our data structure with the submodular function maximization problem.Restatement of Theorem~\ref{thm:main_formal}] \label{thm:app_main_formal}
There is a submodular algorithm such that, given a submodular function maximization problem with cardinality constraint and two precision parameters $\epsilon,\delta$ as input, it runs in $O(\epsilon^{-2}(nd+kn+kd^2)\log(k/\delta))$ time and returns a solution $S$ with $f(S) \ge (1-1/e)\max_{|S|=k}\{f(S)\} - k(2-1/e)\epsilon$  with probability at least $1-\delta$.
\end{theorem}

\begin{proof}

Combining Corollary~\ref{cor:error_approximate_greedy_algo} (the correctness of Algorithm~\ref{alg:fast_submodular}) and time complexity analysis of operations in Section~\ref{sec:app_data_structure}, we prove the conclusion.
\end{proof}

Corollary \ref{cor:error_approximate_greedy_algo} establishes that an approximate greedy algorithm using an orcale with error $\epsilon$ achieves a solution $S_{k}$ satisfying 
\begin{align*}
    f(S_{k}) \ge (1 - \frac{1}{e}) \opt - k(2 - 1/e) \epsilon
\end{align*}
where $\opt = \max_{|T|=k}\{f(T)\}$. In Theorem \ref{thm:main_informal}, the informal guarantee simplifies this to
\begin{align*}
    f(S) \ge (1 - \frac{1}{e}) \opt - k.
\end{align*}
The additive error $-k$ or $-k(2 - 1/e ) \epsilon$ raises concerns about achieving a constant factor approximation. To assess when the additive error is negligible, consider whether the guarantee matches a purely multiplicative bound:
\begin{align*}
    (1 - \frac{1}{e}) \opt - k 
    \ge & ~ (1 - \frac{1}{e} - \epsilon) \opt  \\
    \epsilon \opt \ge & ~ k.
\end{align*}
In practical submodular optimization tasks such as sensor placement, influence maximization, or document summarization the objective function value often grows approximately linearly with the number of selected items.

Since the embedding in Equation~\ref{eqn:decompose_f_3} is  an exact isomorphism of the original submodular function, the dimension $d$ does not alter the function's submodularity and therefore has no effect on the $(1 - 1/e)$ approximation ratio. Dimension influences only the arithmetic cost of evaluating the quadratic form $u_i h(S) u_i$, and thus appears exclusively in the running time bound.

\end{document}